%% file: main.tex
\documentclass[pdflatex,sn-apa,iicol]{sn-jnl}

\usepackage{adjustbox}
\usepackage{graphicx}%
\usepackage{amsmath,amssymb,amsfonts}%
\usepackage{amsthm}%
\usepackage{mathrsfs}%
\usepackage[title]{appendix}%
\usepackage{xcolor}%
\usepackage{textcomp}%
\usepackage{manyfoot}%
\usepackage{booktabs}%
\usepackage{multirow}
\usepackage{algorithm}%
\newsavebox{\ijcvtablebox}
\usepackage{algorithmic}%
\usepackage{listings}%
\usepackage{xspace}
\usepackage{colortbl}
\usepackage{wrapfig}
\usepackage{bm}
\usepackage{makecell}
\usepackage{bbding}
\usepackage{float}

\newcommand{\eg}{{\emph{e.g.}}\xspace}

\newcommand{\vs}{{\emph{v.s.}}\xspace}

\usepackage[switch]{lineno}
\usepackage{pifont}
\newcommand{\xmark}{\ding{55}\xspace}%

\usepackage[table]{xcolor}

\usepackage{capt-of}
\theoremstyle{plain}

\usepackage{makecell}

\theoremstyle{thmstyleone}%
\newtheorem{theorem}{Theorem}
\theoremstyle{thmstyletwo}%

\theoremstyle{thmstylethree}%

\newtheorem{lemma}[theorem]{Lemma}
\newtheorem{corollary}[theorem]{Corollary}
\begin{document}
\AtBeginDocument{%
  \setlength{\bibsep}{0pt}%
}
\renewcommand{\bibfont}{\small}
\title[Feature-Space Smoothing]
{Feature-Space Smoothing: Certified Robustness of Deep Representations}


\author[1]{\fnm{Song} \sur{Xia}}\email{xias0002@ntu.edu.sg}

\author[1]{\fnm{Meiwen} \sur{Ding}}\email{ding0159@ntu.edu.sg}

\author[2]{\fnm{Chenqi} \sur{Kong}}\email{cqkong@nus.edu.sg}

\author[3]{\fnm{Wenhan} \sur{Yang}}\email{yangwh@pcl.ac.cn}

\author[1]{\fnm{Xudong} \sur{Jiang}}\email{exdjiang@ntu.edu.sg}


\affil[1]{\orgdiv{ROSE Lab, School of EEE}, \orgname{Nanyang Technological University}, \orgaddress{ \city{Singapore}, \postcode{639798}}}

\affil[2]{\orgdiv{NAII Lab, School of Computing}, \orgname{National University of Singapore}, \orgaddress{ \city{Singapore}, \postcode{119391}}}

\affil[3]{\orgdiv{Peng Cheng Laboratory}, \orgaddress{ \city{Shenzhen}, \postcode{518055}, \state{Guangdong}, \country{China}}}


\abstract{Modern deep learning models exhibit strong capabilities across diverse applications, yet remain vulnerable to malicious inputs that induce erroneous predictions via feature-space distortion. 
To address that, we propose Feature-Space Smoothing (FS), a representation-level certification framework built upon Gaussian randomized smoothing.
FS converts a given feature encoder into a smoothed variant that is guaranteed to maintain a certified lower bound on the cosine similarity between clean and adversarial features under $\ell_2$-bounded perturbations.
We then establish that this Feature Cosine Similarity Bound (FCSB) can be extended to the prediction-wise certification under the cosine similarity measure, and the value of FCSB is determined by the encoder’s intrinsic Gaussian robustness score.
Building on those insights, we introduce the Gaussian Smoothness Booster (GSB), a plug-and-play module to improve the encoder's Gaussian robustness score. 
Specifically, the GSB module is plugged to enhance the feature-space consistency and maintain the feature utility for downstream tasks under Gaussian perturbations.
This design enables seamless integration of FS on the protected model, e.g., Multimodal Large Language Models (MLLMs), without additional model retraining or alignment, improving its robustness while preserving the performance for downstream task-oriented decoding. 
Extensive experiments demonstrate that integrating FS consistently provides non-trivial certified robustness and significantly improves task-oriented performance under strong white-box adversarial attacks across diverse models and applications.}




\maketitle

\input{math_commands}

\section{Introduction}
\vspace{-1mm}
\label{sec:intro}
The emergence of the Multimodal Large Language Models (MLLMs), such as GPT-5.4~\citep{openai2026gpt54}, Gemini 3.1 Pro~\citep{deepmind2026gemini31}, and Claude Sonnet 4.6~\citep{anthropic2026claude46}, has fundamentally reshaped existing working paradigms and significantly advanced societal productivity.
Despite their remarkable capabilities across a broad spectrum of real-world tasks, these models still encounter critical safety challenges, such as adversarial vulnerabilities~\citep{chen2026adaea++,cui2024robustness,li2025frustratingly,jia2025adversarial,wei2023adaptive,wei2024physical,ding2026transferable,li2023recognizing,zhang2026nap,liu2026sega,wei2026enhancing,hu2025dynamicpae,peck2023introduction}.
Adversaries can manipulate predictions of deep learning models to a malicious state by injecting subtle and imperceptible perturbations to the clean inputs, exploiting the models’ insufficient local smoothness and uncontrolled Lipschitz continuity~\citep{goodfellow2014explaining, cohen2019certified, hein2017formal,xiamitigating}.

Countermeasures towards those threats can be roughly classified into empirical defense and certified defense.
Typical empirical approaches include adversarial training~\citep{madry2018towards,rebuffi2021data,wang2024revisiting,xhonneux2024efficient,casper2024defending,li2025aroid,lehner20243d,maounderstanding,angioni2025robustness,lau2023interpolated} and input purification~\citep{nie2022diffusion,yoon2021adversarial,lee2023robust,lei2025instant,zollicoffer2025lorid}. 
Despite their demonstrated empirical effectiveness, these approaches lack formal robustness protection and remain susceptible to stronger adversaries~\citep{tramer2020adaptive,chen2023adaptive,wu2020stronger,olivier2023many}.
Moreover, the multimodal nature of MLLMs also poses a great challenge for existing adversarial training methods.
Since they accept heterogeneous inputs across diverse domains, training a robust encoder via adversarial training that can generalize to various scenarios is challenging and computationally costly.
In contrast, certified approaches aim to guarantee that the model returns a constant prediction result within a certain range, usually a $\ell_2$ or $\ell_\infty$-norm constrained area~\citep{raghunathan2018certified,wong2018provable,hao2022gsmooth,kakizaki2023certified,cohen2019certified,xiamitigating,salman2019provably}.
Although randomized smoothing~\citep{cohen2019certified} has been extended beyond classification, many existing certification methods~\citep {xiamitigating, wang2024drf,hao2022gsmooth,salman2019provably} remain tailored to specific prediction rules or output structures, making it challenging to obtain reusable robustness guarantees for shared feature encoders across diverse downstream tasks.

\begin{figure}[t]
    \centering
    \includegraphics[width=1\linewidth]{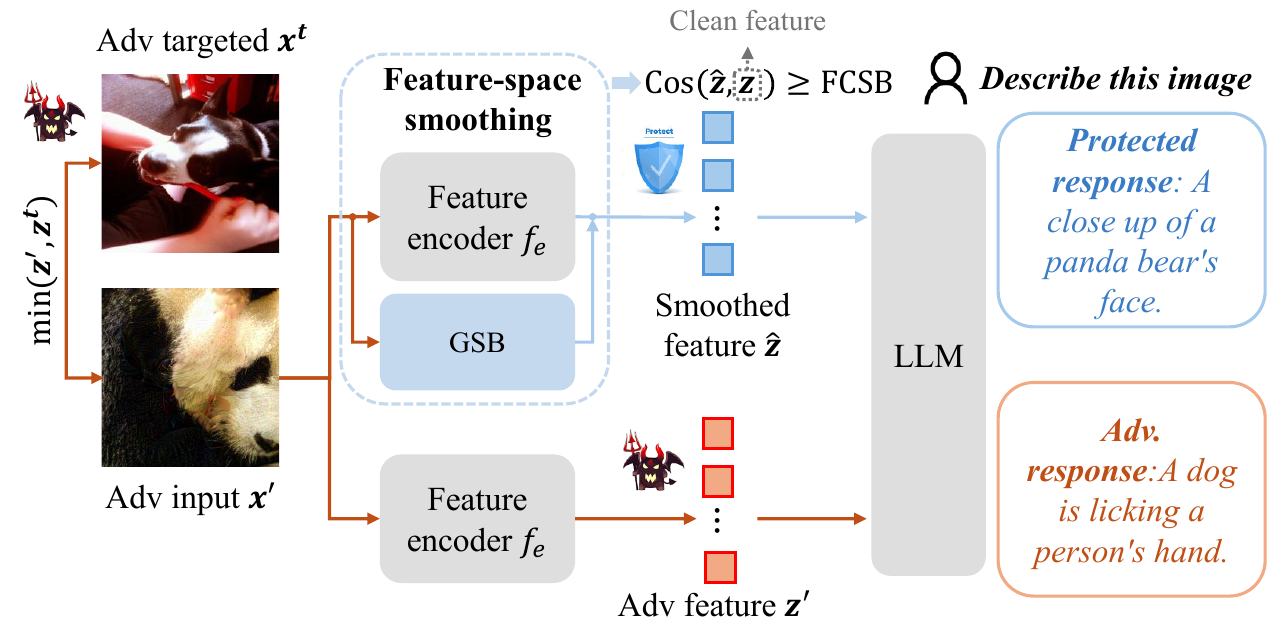}
    \vspace{-1mm}
    \caption{Illustration of the FS framework with GSB, which guarantees that the cosine similarity of the adversarial and clean features extracted by the MLLM's encoder is larger than FCSB for robust predictions.}
    \label{fig:teaser}
\end{figure}

This raises a fundamental challenge: \textbf{Developing an effective defense framework that simultaneously provides certified robustness while remaining generalized across diverse models and application scenarios.}
To address that, we propose Feature-space Smoothing (FS), which applies Gaussian randomized smoothing to feature representations and derives cosine-similarity guarantees that can be utilized by compatible downstream prediction heads.
Specifically, FS first smooths the vanilla feature encoder, and the resulting smoothed encoder is then guaranteed to maintain a provable Feature Cosine Similarity Bound (FCSB), a lower bound on the cosine similarity between clean and adversarial feature representations under $\ell_2$-norm bounded attacks.
Notably, we show that the FCSB naturally extends to prediction-level certification under the cosine similarity measure. 
We further establish that the FCSB achieved by a smoothed feature encoder is intrinsically determined by the Gaussian robustness of the original encoder. 
This motivates the Gaussian robustness score, a differentiable metric that effectively bridges the Gaussian robustness of the original encoder and the FCSB of its smoothed counterpart.
%

Generally, the Gaussian robustness score of vanilla encoders remains limited without Gaussian noise augmented training, leading to a correspondingly suboptimal FCSB.
To address this and make the proposed FS framework generalized to large-scale pre-trained models like MLLMs, we propose the Gaussian Smoothness Booster (GSB), which consists of a lightweight Gaussian denoiser and a residual feature smoothness mapper.
This GSB serves as a plug-and-play module to enhance the Gaussian robustness score of the protected encoders while avoiding the costly fine-tuning or alignment.
Specifically, the denoiser is plugged prior to the feature encoder and is trained to mitigate Gaussian corruption.
In parallel, the smoothness mapper operates after feature extraction for feature-level refinement.
It is trained to preserve the original feature distribution while further enhancing the Gaussian robustness score.
These two components of GSB work synergistically to enhance the certified robustness of the protected encoder.
The GSB is trained using a proposed utility–robustness trade-off loss in a self-supervised manner, aiming to improve the Gaussian robustness without modifying the backbone encoder.

To comprehensively validate the proposed FS framework, we apply it to various vision encoders, including CLIP~\citep{radford2021learning}, SigLIP 2~\citep{tschannen2025siglip}, and EVA-CLIP~\citep{sun2023eva}, and systematically evaluate the resulting certified robustness in both feature-wise and prediction-wise results. 
%
%
Beyond theoretical certification, we further evaluate FS on MLLMs against state-of-the-art white-box adversarial attacks tailored to multimodal models, and compare it with strong adversarial training baselines. 
Extensive experimental results show that FS not only establishes strong certified robustness guarantees but also yields superior empirical adversarial robustness across diverse MLLMs and downstream tasks. 
Overall, the main contributions of this work are:
\begin{itemize}
    \item We develop Feature-Space Smoothing (FS), a representation-level certification framework based on Gaussian randomized smoothing, and derive a certified FCSB between clean and adversarial feature representations.
    \item We introduce Gaussian Smoothness Booster (GSB), a lightweight plug-and-play module that combines Gaussian denoising and residual feature refinement to improve the certified robustness of frozen pretrained encoders without modifying their backbone parameters.
    \item We analyze how feature-level cosine similarity guarantees can be translated into prediction-level robustness certificates for cosine-similarity-based tasks, including classification and image--text retrieval.
    \item We conduct extensive experiments demonstrating the non-trivial practical protection offered by FS. Applying the proposed FS to various vision-language encoders and MLLMs substantially improves their task-oriented performance under diverse white-box attacks across multiple applications, without requiring retraining or fine-tuning of the protected models.
\end{itemize}

\section{Related Work}
\label{sec:related work}
\noindent\textbf{Adversarial attacks on MLLMs.} 
While MLLMs continue to achieve remarkable performance across diverse applications, extensive works~\citep{qi2024visual,cui2024robustness,zhao2023evaluating,jia2025adversarial,li2025frustratingly,wang2024break,zhang2024adversarial,zhang2025anyattack,xie2025chain} have exposed their adversarial vulnerabilities, raising serious safety concerns.
Early work, such as AttackVLM~\citep{zhao2023evaluating}, explores transferable attacks by disrupting the feature representations of CLIP~\citep{radford2021learning} and BLIP~\citep{li2023blip}, showing strong adversarial transferability among open-source models but limited effectiveness against closed-source commercial systems.
More recent approaches, such as M-Attack~\citep{li2025frustratingly} and FOA-Attack~\citep{jia2025adversarial}, further advance this direction by leveraging multi-extractor ensembles and feature-space alignment, achieving over 90\% targeted attack success rates on image-captioning tasks against powerful closed-sourced commercial MLLMs (\eg, ChatGPT-4o).

Beyond the single modality manipulations, recent work has shifted toward multimodal jailbreak and cross-modality safety-alignment attacks~\citep{li2025odysseus,zhao2025vrsa,shayeganijailbreak,fang2025safemlrm}, where adversarial images or image–text pairs bypass alignment safeguards to elicit harmful outputs. Such attacks exploit the interaction between vision encoders and aligned LLMs~\citep{shayeganijailbreak} or amplify malicious intent via carefully crafted visual inputs~\citep{li2024images}, and can further propagate in agentic or multi-agent settings as “infectious jailbreaks”~\citep{gu2024agent}. These developments underscore the urgent need for trustworthy defenses that provide robust and provable protection for MLLMs.

\noindent\textbf{Adversarial defense on MLLMs.}
Adversarial defense methods can be broadly classified into {empirical} and {provable} approaches.
Empirical defenses for MLLMs mainly include {adversarial training}~\citep{madry2018towards,rebuffi2021data,wang2024revisiting,xhonneux2024efficient,casper2024defending,schlarmann2024robust,malik2025robust,maounderstanding}, which enhances robustness by augmenting training data with adversarial examples, and {input purification}~\citep{nie2022diffusion,lei2025instant,li2025adbm}, which employs generative mechanisms such as diffusion models or autoencoders to recover clean inputs prior to inference.
Recent studies~\citep{malik2025robust, maounderstanding,schlarmann2024robust} have revealed that utilization of adversarially trained CLIP-feature encoders can enhance adversarial robustness for MLLMs.
Complementary to training-time approaches, a growing body of work investigates inference-time and lightweight defenses.
These methods aim to enhance safety without retraining the full model. Representative approaches include prompt- or guardrail-based techniques that steer model behavior~\citep{wang2024adashield,oh2025uniguard}, inference-time alignment guided by safety objectives~\citep{ghosal2025immune}, and representation-level interventions such as activation steering and token filtering~\citep{wang2025steering,chen2025safeptr}. In addition, detection-based methods~\citep{zhang2024pip} and cross-modality defenses that introduce protective perturbations have also been explored~\citep{li2025attack}.
Despite progress, these empirical defenses typically lack formal robustness guarantees and can remain brittle under adaptive or unseen attacks~\citep{tramer2020adaptive,chen2023adaptive,wu2020stronger}. Moreover, adversarial training can be computationally expensive and may degrade clean performance.

\noindent\textbf{Certified defense.} certified defenses aim to provide mathematically provable robustness guarantees.
The use of Gaussian smoothing for certified robustness was initially introduced for classification models~\citep{cohen2019certified,li2018second,lecuyer2019certified}, yet its theoretical formulation is restricted to one-dimensional outputs, limiting its applicability to tasks such as auto-regression or multimodal generation.

To overcome this limitation, prior work has explored extending certification beyond classification. These efforts include certifying structured outputs in metric spaces~\citep{kumar2021center}, adapting smoothing to regression and localization tasks such as object detection~\citep{chiang2020detection}.
More recently, certified defenses have been adapted to vision-language models. Existing methods either specialize in smoothing to CLIP-style predictors through incremental smoothing and caching~\citep{nirala2024fast}, or introduce prompt-based adaptations for domain-specific VLMs~\citep{hussein2024promptsmooth}. For generative models, smoothing has been applied by reducing sequence outputs to auxiliary classification objectives (e.g., harmful vs.\ harmless)~\citep{seferis2025randomized}. Despite these advances, current approaches remain output-centric, requiring task-specific reductions or surrogate objectives, which limits their generality across diverse model pipelines.

Prior studies, like Center Smoothing~\citep{kumar2021center} and Smoothed Embeddings~\citep{pautov2022smoothed}, have also explored randomized smoothing for structured outputs and continuous embeddings. Unlike center-based certification or task-specific embedding certification, our approach focuses on bounding the cosine similarity between the Gaussian-smoothed representation and the original clean representation, enabling its use in different downstream cosine-based prediction settings.
Furthermore, FS operates without retraining or modifying the prediction head, enabling seamless integration into existing MLLMs while preserving feature utility.

\section{Feature-Space Smoothing}
\label{sec: feature space smoothing}
\subsection{Preliminary}
Let $\mathcal{F}$ denote a general deep learning model consisting of a feature encoder $f_e: \vx \rightarrow \vz$ that encodes the input $\vx$ to a feature representation $\vz$, and a decoder $f_d: \vz \rightarrow \vy$ that produces the final output $\vy$.
Let $\mathcal{L}$ denote the general loss function (e.g., cross-entropy) that measures the discrepancy between the model's output and the ground truth.

\noindent\textbf{Adversarial attacks:}
Let $\mathcal{B}_{\epsilon}(\vx)=\left\{ {\vx':{{\left\| {\vx' - \vx} \right\|}_p} \le \epsilon } \right\}$ be {an} $\ell_p$-norm ball centered at the input $\vx$, where $\epsilon$ is a pre-defined perturbation bound. 
For each input $\vx$, the adversarial attacks aim to find an adversarial input $\vx'=\vx+{\bm{\delta}}$ that misleads the model by solving: 
\begin{equation}
%
\mathop {\max }\limits_{{\vx} + {\bm{\delta}} \in \mathcal{B}_{\epsilon}(\vx)} \mathcal{L}\left( {\mathcal{F}\left( {{\vx}} \right),\mathcal{F}\left( {{\vx} + {\bm{\delta}}} \right)} \right).
%
  \label{eq:1}
\end{equation}
\noindent\textbf{Adversarial effects on feature representations:} While adversarial attacks primarily aim to alter the model’s predictions, numerous studies~\citep{jia2025adversarial,xia2024transferable,li2020yet,huang2019enhancing,li2023improving,ding2024transferable} have shown that successful attacks typically induce substantial distortions in the model’s feature representations.
Let $\vx'$ denote the adversarial example with adversarial feature $\vz'$.
Let $\vz^{t}$ be the adversarial targeted feature with malicious semantic meaning.
The attack generally leads to $\max \mathcal{L}\left( {\vz',\vz} \right)$ for untargeted attacks and $\min \mathcal{L}\left( {\vz',\vz^{t}} \right)$ for targeted attacks.
Thus, ensuring a robust feature encoder that $ \min \mathcal{L}\left( {\vz',\vz} \right)$ is crucial for the trustworthy prediction.

\noindent\textbf{Randomized smoothing:} Consider a $k$ classes classification problem with the input $\vx \in {\mathbb{R}^d}$ and the label $ \vy \in \mathcal{Y} = \left\{ {{c_1}, \ldots,{c_k}} \right\}$.
Randomized Smoothing (RS) first corrupts each input $\bm{x}$ by adding the Gaussian noise $\bm{\varepsilon} \sim \mathcal{N}(0,{\sigma ^2}\bm{I})$. 
Then it turns an arbitrary base classifier $\mathcal{F}$ into a smoothed version $\hat{\mathcal{F}}$ that possesses ${\ell_2}$ certified robustness guarantees. 
The smoothed classifier $\hat{\mathcal{F}}$ returns whichever class the base classifier $ \mathcal{F}$ is most likely to return among the distribution $\vx+\bm{\varepsilon}\sim \mathcal{N}(\vx,{\sigma ^2}\bm{I})$, which is:
\begin{equation}
  \begin{array}{l}
\hat{\mathcal{F}}(\vx) = \mathop {\arg \max }\limits_{c \in \mathcal{Y}} \sP(\mathcal{F}(\vx + \bm{\varepsilon} ) = c).
\end{array}
  \label{eq:RS_classify}
\end{equation}
RS then guarantees a certified radius $\mathcal{R}$ for this smoothed classifier $\hat{\mathcal{F}}$.
For any perturbation $\delta$ satisfying $\left \| \delta \right \|_2 \le \mathcal{R}$, the smoothed classifier is guaranteed to return a robust prediction that makes $\hat{\mathcal{F}}(\vx+\delta)=F(x)$.

\noindent\textbf{Limitations for RS.} Although randomized smoothing (RS) provides effective certified protection, its theoretical framework is inherently tied to specific prediction settings, such as image classification, which limits its applicability to broader feature spaces and multimodal tasks.
Meanwhile, estimating $\mathbb{P}(\mathcal{F}(\vx + \bm{\varepsilon}) = c)$ in Equation~\ref{eq:RS_classify}, incurs substantial computational overhead, since each estimation requires multiple forward passes through the entire model.

\subsection{Certified Bound via Feature Space Smoothing}

Considering the limitations inherent in the RS, we introduce the Feature-space Smoothing (FS), a general framework for feature-wise robustness certification.
By turning a given encoder $f_e$ into a smoothed version $\hat{f_e}$, FS guarantees that $\hat{f_e}$ maintains a certified lower bound on the cosine similarity between clean and adversarial representations under $\ell_2$-norm constrained perturbations.

\noindent\textbf{Smoothed feature encoder.}
For any feature encoder $f_e: \vx \rightarrow \vz$, where $\vz$ is the representation normalized into the $l_2$ unit sphere, FS defines the smoothed encoder $\hat f_e(\bm{x})$ as:
\begin{align}
&\hat f_e(\bm{x}) 
= \mathbb{E}_{\bm{\varepsilon} \sim \mathcal{N}(0, \sigma^2I)} [f_e(\bm{x} + \bm{\varepsilon})] \notag\\
&= \frac{1}{(2\pi\sigma^2)^{d/2}}
   \int_{\mathbb{R}^d} 
   f_e(\bm{x} + \bm{\varepsilon})
   \exp\!\left(-\frac{\|\bm{\varepsilon}\|_2^2}{2\sigma^2}\right)
   d\bm{\varepsilon},
\label{eq:apdx_RS}
\end{align}
where $I$ denotes the $d\times d$ identity and $\|\cdot\|$ is the Euclidean norm. 
$\bm{\varepsilon} \sim \mathcal{N}(0, \sigma^2I)$ denotes Gaussian noise with zero mean and standard deviation $\sigma$.
The smoothed feature encoder $\hat f_e(\bm{x}) $ outputs the expectation of feature representations over the Gaussian distribution $\mathcal{N}(\vx, \sigma^2I)$.

\noindent\textbf{Gaussian robustness score.} Define ${S_{\bm{x_t}}}(\vx)$ as the score function that evaluates feature discrepancy between an input $\vx$ and a targeted example $\vx_t$, which is:
\begin{equation}
S_{\vx_t}(\vx)
=\tfrac{1}{2}\Big(1+\mathrm{Cos}\big(f_e(\vx),\,f_e(\vx_t)\big)\Big),
\label{eq:score_func_cos}
\end{equation}
where ${Cos}(\cdot,\cdot)$ denotes the cosine similarity and $S_{\vx_t}(\vx)\in [0,1]$.
The Gaussian robustness score $\hat{S}(\vx)$ is defined as:
\begin{equation}
\begin{aligned}
&\hat S(\vx)
=\mathop {\mathbb{E}}\limits_{\bm{\varepsilon}\sim\mathcal N(0, \sigma^2I)}
\big[S_{\vx}(\vx+\bm{\varepsilon})\big]\\[-3pt]
&=\tfrac{1}{2}\!\left(
1+\mathop {\mathbb{E}}\limits_{\bm{\varepsilon}\sim\mathcal N(0, \sigma^2I)}
\!\left[\mathrm{Cos}\big(f_e(\vx+\bm{\varepsilon}),\,f_e(\vx)\big)\right]\right).
\label{eq:smoothed_score_cos}
\end{aligned}
\end{equation}
The score $\hat{S}(\vx)$ evaluates the expected cosine similarity between the feature representation of a Gaussian-perturbed input $\vx+\bm{\varepsilon}$ and that of the clean input $\vx$, characterizing the Gaussian robustness of the vanilla feature encoder ${f}_e$.
Building upon the Gaussian smoothing properties of bounded score functions~\citep{salman2019provably}, we derive a certified cosine similarity bound for continuous feature representations.

\begin{theorem}[\textbf{Certified lower bound on the adversarial feature cosine similarity}]
\label{theorem:1}
For a feature encoder ${f}_e$ and its smoothed variant $\hat{f}_e$, let $\vx$ and $\vx'$ be clean and adversarial inputs with $\|\vx'-\vx\|\le \epsilon$. If $\Phi^{-1}(\hat S(\vx))>\frac{\epsilon}{\sigma}$, the cosine similarity between the adversarial feature $\hat{f}_e(\vx')$ and clean feature ${f}_e(\vx)$ satisfies: $\mathrm{Cos}(\hat f_e(\vx'),f_e(\vx))\ge 2{\Phi } ({\Phi ^{ - 1}}(\hat{S}(\vx))- \frac{\epsilon}{\sigma})-1$.
\end{theorem}
\noindent Denote $2{\Phi }  \left ({\Phi ^{ - 1}}(\hat{S}(\vx))- \frac{\epsilon}{\sigma} \right)-1$ as the Feature Cosine Similarity Bound (FCSB).
Theorem~\ref{theorem:1} establishes an explicit relationship between the Gaussian robustness score of the original encoder and the FCSB of its smoothed version.
The proof of Theorem~\ref{theorem:1} is in the Appendix, Section~\ref{sec: proof of theorem1}.

\begin{corollary}[\textbf{Certified radius $\mathcal{R}$ for adversarial cosine similarity $\ge$ 0.5}]
\label{corollary:cos-radius}
Let $\vx$ be the clean input, and $\vx'$ be the adversarial input. Then $\mathrm{Cos}(\hat{f}_e(x'),{f}_e(x)) \ge 0.5$, for all $\vx'$ with $\|\vx'-\vx\|_2\le \mathcal{R}$, where:
\begin{equation}
\mathcal{R}
=\sigma(\Phi^{-1}\!\big(\hat S(\vx)\big)
-\Phi^{-1}(0.75)).
\label{eq:certified-radius}
\end{equation}
\end{corollary}
Building upon Theorem~\ref{theorem:1}, Corollary~\ref{corollary:cos-radius} establishes a certified radius $\mathcal{R}$ for the smoothed feature encoder.
For any adversarial perturbation satisfying $\|\delta\|_2 \le \mathcal{R}$, the FCSB of the smoothed encoder is guaranteed to remain above 0.5.
Overall, Theorem~\ref{theorem:1} reveals that:
\begin{itemize}

    \item By FS, we can turn any given feature encoder $f_e$ into a smoothed version $\hat{f}_e$ that maintains a FCSB between the adversarial and clean feature representations.

    \item By maximizing the robustness score $\hat{S}(\vx)$ of the given feature encoder $f_e$, we can effectively enhance the value of FCSB derived on its smoothed version.

\end{itemize}
\section{Prediction-wise Certification via FS}
Beyond feature-level certification, the guarantees provided by FS can be propagated to the prediction space by applying smoothing to the deepest-layer feature representations.
Predictions are then produced via an analytically solvable decoding head (\eg, linear or cosine similarity-based head), which enables a closed-form mapping from feature-space certification to prediction-level robustness guarantees.
\subsection{Case Study: Certified Radius on Top-1 Retrieval}
\label{sec:clip_retrieval}
Consider the standard top-1 image--text retrieval setting over a fixed normalized candidate text embeddings $\{\vt_k\}_{k=1}^K$.
Let $\vz=f_e(\vx)$ and $\hat\vz'=\hat f_e(\vx')/\|\hat f_e(\vx')\|_2$ denote the $\ell_2$-normalized clean and smoothed adversarial features, respectively.
For a given input $\vx$, the retrieval score $s_k(\vz)$ for candidate text $k$ is given by the cosine similarity, which is:
\begin{equation}
   s_k(\vz) = \vz^\top \vt_k.
\end{equation}
The index of the clean top-1 retrieved text is therefore:
\begin{equation}
    y=\arg\max_{k\in\{1,\dots,K\}} s_k(\vz),
\end{equation}
and we assume that the clean top-1 retrieval result is unique.
Define the text--text correlation between text embeddings $\vt_y$ and $\vt_k$ as $\rho_{y,k}$:
\begin{equation}
    \rho_{y,k}=\vt_y^\top \vt_k.
\end{equation}
Let $a_y= s_y(\vz), \quad a_k= s_k(\vz).$
For each competitor $k\neq y$, based on the geometry of pairwise retrieval boundaries, we can derive a closed-form solution $\gamma_{y,k}$ for the largest cosine similarity between $\vz$ and any unit vector on the pairwise decision boundary between $\vt_y$ and $\vt_k$:
\begin{equation}
\label{eq:gamma-pairwise}
    \gamma_{y,k}
    =
    \sqrt{
    1-\frac{(a_y-a_k)^2}{2(1-\rho_{y,k})}
    }.
\end{equation}
Detailed explanation of Equation~\ref{eq:gamma-pairwise} can be found in the Appendix Section~\ref{app:retrieval-cert-proof}.
Equivalently, if $\vz^\top \hat\vz' > \gamma_{y,k}$, then the pairwise ranking between $\vt_y$ and $\vt_k$ for adversarial feature $\hat\vz'$ is preserved, i.e.,$ \hat\vz'^\top \vt_y > \hat\vz'^\top \vt_k.$
We therefore define the top-1 critical cosine threshold as
\begin{equation}
    \gamma_{\mathrm{top1}}^*
    =
    \max_{k\neq y} \gamma_{y,k}.
\end{equation}
Thus, whenever $\vz^\top \hat\vz' > \gamma_{\mathrm{top1}}^*$, the clean top-ranked text $\vt_y$ remains more similar than every competing text under the adversarial feature $\hat\vz'$. 

\begin{figure*}[t]
    \centering
   \includegraphics[width=0.98\linewidth]{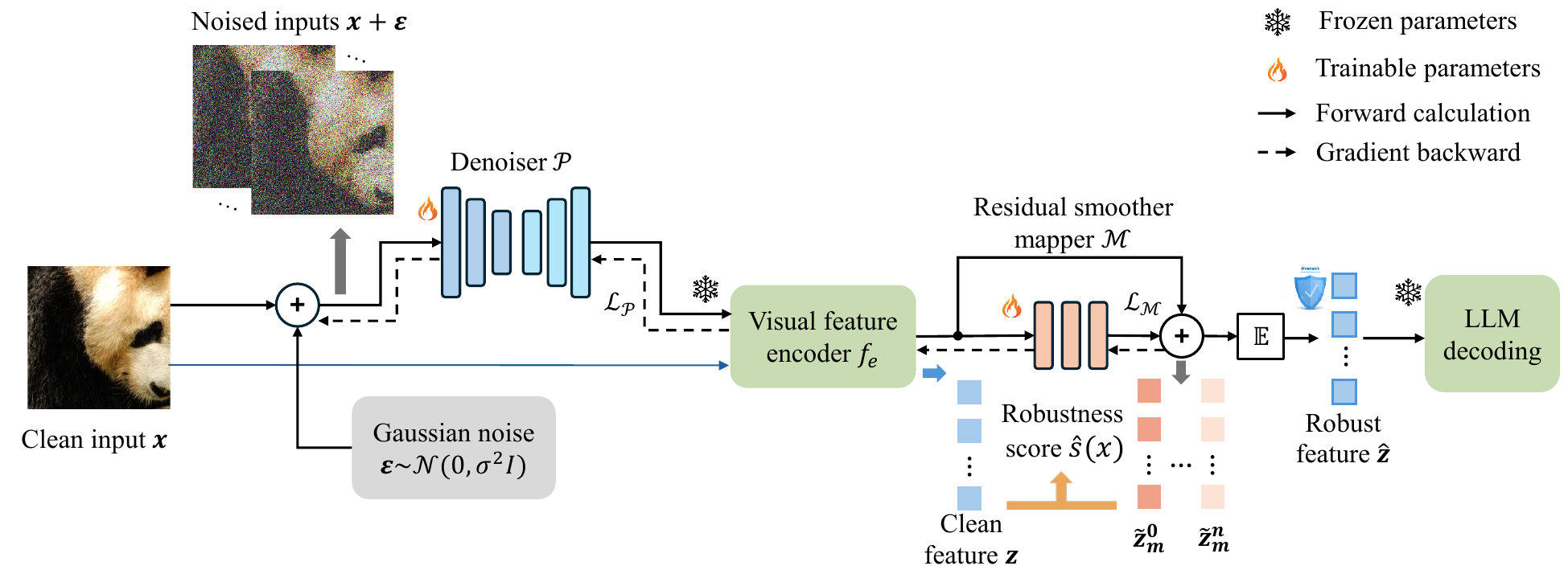}
   \caption{The training framework of the GSB. The denoiser performs pre-processing, and the smoothness mapper refines post-extracted features to enhance the Gaussian robustness. Parameters of MLLMs are frozen, and the denoiser and mapper are optimized with $\mathcal{L}_\mathcal{P}$ and $\mathcal{L}_\mathcal{M}$. For evaluation, $\vx$ is replaced with adversarial input $\vx'$, and the forward calculation marked in the blue line will be removed. }
   \label{fig:flow_chat_df}
\end{figure*}

\begin{corollary}[\textbf{FS Certified radius $\mathcal{R}_{\mathrm{ret}}$ on top-1 retrieval}]
\label{thm:pred-cert-retrieval}
Let $y=\arg\max_k \vz^\top \vt_k$ denote the clean top-1 retrieved text.
For any adversarial input satisfying $\|\vx'-\vx\|_2 < \mathcal{R}_{\mathrm{ret}}$, the top-1 retrieval result on the smoothed adversarial feature is preserved, i.e., $\arg\max_k \hat\vz'^\top \vt_k = y$, where
\begin{equation}
\label{eq:eps-condition-retrieval}
    \mathcal{R}_{\mathrm{ret}}
    =
    \sigma\left(\Phi^{-1}(\hat S(\vx))
    -
    \Phi^{-1}\!\left(
    \frac{1+\gamma_{\mathrm{top1}}^*}{2}
    \right))\right.
\end{equation}
\end{corollary}
Corollary~\ref{thm:pred-cert-retrieval} is an extension of Corollary~\ref{corollary:cos-radius} on the top-1 retrieval setting.
Detailed proof of Corollary~\ref{thm:pred-cert-retrieval} can be found in the Appendix Section~\ref{app:retrieval-cert-proof}.
It indicates that the FS-certified retrieval radius is jointly determined by the intrinsic robustness score $\hat S(\vx)$ and the geometric structure of the clean retrieval configuration.
Specifically, a larger Gaussian robustness score $\hat S(\vx)$ and a greater margin separating the retrieved text from all competing texts lead to a larger certified radius.

\subsection{General Prediction-level Cosine Bound via FS}
\label{sec:general-target-cert}

The top-1 retrieval result in Section~\ref{sec:clip_retrieval} is obtained by combining the FCSB with the geometry of pairwise retrieval boundaries.
\textbf{{Generally, the FS guarantee can be propagated to tasks where the downstream prediction can be represented as a normalized vector and the adversarial prediction is guaranteed to remain within a bounded cosine distance from the clean prediction.}}
Specifically, for an input $\vx$, let $\vz$ denote the normalized clean prediction and let $\vt_y$ denote the normalized target prediction. 
Let $a_y = \vz^\top \vt_y$ be the cosine similarity between the clean prediction and the target prediction and $\gamma$ be the FCSB derived in Theorem~\ref{theorem:1}.

\begin{theorem}[\textbf{General target-level cosine similarity bound under worst-case feature similarity}]
\label{thm:general-target-bound}
For adversarial input satisfying $\|\vx'-\vx\|_2\le \epsilon$, consider the worst-case boundary scenario where the FCSB is attained, i.e., $\mathrm{Cos}(\hat\vz',\vz)=\gamma$, the cosine similarity between the normalized adversarial prediction $\hat\vz'$ and the targeted prediction $\vt_y$ satisfies:
\begin{equation}
\label{eq:general-target-bound-compact}
    \left|
    \hat\vz'^\top \vt_y-\gamma a_y
    \right|
    \le
    \sqrt{(1-\gamma^2)(1-a_y^2)}.
\end{equation}
\end{theorem}
Proof of Theorem~\ref{thm:general-target-bound} can be found in the Appendix, Section~\ref{app:general-prediction-cosine-bound}.
It provides a target-level guarantee on the adversarial prediction.
For any perturbation satisfying $\|\vx'-\vx\|_2\le\epsilon$, the cosine similarity between the normalized adversarial prediction $\hat{\vz}'$ and the target prediction $\vt_y$ is lower-bounded by $\gamma a_y-\sqrt{(1-\gamma^2)(1-a_y^2)}$.
Consequently, both a larger FCSB $\gamma$ and a larger clean target similarity $a_y$ tighten the admissible range of $\hat{\vz}'^\top \vt_y$, thereby making the target prediction more stable under adversarial perturbations.
In particular, when either $\gamma$ approaches $1$ or $a_y$ is close to $1$, the uncertainty interval shrinks, implying that the adversarial prediction is forced to remain close to the clean target prediction.
Therefore, Theorem~\ref{thm:general-target-bound} provides a geometric characterization of the admissible target-level cosine similarity under the worst-case feature similarity condition.

\noindent\textbf{Discussion on potential application scenarios.}
Theorem~\ref{thm:general-target-bound} applies to any downstream task whose prediction can be represented in a normalized embedding space.
In such settings, FS certifies the similarity preserved between the adversarial prediction and the desired target embedding.
Potential scenarios include Top-k retrieval tasks, multi-target prediction, and dense prediction problems whose outputs can be flattened and normalized into a common embedding space.
Thus, the feature-space certification provided by FS naturally extends to a broad family of prediction-level robustness guarantees beyond the specific top-1 retrieval case considered above.

\section{Gaussian Smoothness Booster}
\label{sec: GSB}
Theorem~\ref{theorem:1} and Theorem~\ref{thm:general-target-bound} reveal an intriguing robustness property of the proposed FS framework.
However, feature encoders $f_e$ of most pre-trained large models (\eg, MLLMs) generally exhibit limited Gaussian robustness, which restricts the value of FCSB derived on $\hat{f}_e$.
One solution is to estimate the smoothness score $\hat{S}(\vx)$ through Monte Carlo sampling and re-train $f_e$ to maximize this score $\hat{S}(\vx)$ via gradient backpropagation.
However, this could {{reduce the adaptability and practicality, as fine-tuning and re-aligning the large pre-trained model is highly complex and costly}}.

This thereby motivates us to propose the Gaussian Smoothness Booster (GSB), a plug-and-play module that can be seamlessly integrated with existing models to enhance its Gaussian robustness score $\hat{S}(x)$.
The training framework of our proposed GSB is shown in Figure~\ref{fig:flow_chat_df}.
The denoiser $\mathcal{P}$ operates before feature extraction to denoise the Gaussian perturbations, and the smoothness mapper $\mathcal{M}$ performs post-extraction to do the feature refinement. 
Those two modules work together to enhance the Gaussian robustness of the given feature encoder.

\subsection{Gaussian Denoiser} 
To perform the Gaussian denoising, the denoiser $\mathcal{P}$ is trained to minimize the reconstruction loss $l_{mse}$, defined as:
\begin{equation}
l_{mse}=\mathop{\mathbb{E}}\limits_{\substack{\vx \sim \mathcal{D},~\bm{\varepsilon} \sim \mathcal{N}(0, \sigma^2I)}}\|\vx-\mathcal{P}(\vx+\mathbf{\varepsilon})\|,    
\label{eq: loss mse}
\end{equation}
where $\mathcal{D}$ represents the data distribution. 
Meanwhile, to further enhance robustness score after plugging $\mathcal{P}$, we also introduce a robustness loss $l^{\mathcal{P}}_{rb}$, defined as:
\begin{equation}
l^{\mathcal{P}}_{\mathrm{rb}}
= 1-\mathop{\mathbb{E}}\limits_{\vx \sim \mathcal{D},~ \bm{\varepsilon} \sim \mathcal{N}(0, \sigma^2I)}
\big[\mathrm{Cos}\big(f_e(\mathcal{P}(\vx+\bm{\varepsilon})), f_e(\vx)\big)\big],
\label{eq:loss smooth}
\end{equation}
which encourages feature consistency between the purified and clean representations.
%
%
We then {train} this model on the dataset $\mathcal{D}$ using the purifier loss $\mathcal{L}_\mathcal{P}$, which is:
\begin{equation}
\mathcal{L}_\mathcal{P}=\ l^{\mathcal{P}}_{\mathrm{rb}}+\lambda_1 l_{mse},
\label{eq:loss for purifier}
\end{equation}
where 
$\lambda_1$ is the weighting coefficient. 
In our experiments, we employ a lightweight U-Net architecture with approximately 10M parameters as the denoising module.
More details of the denoiser $\mathcal{P}$ are provided in Supplementary material.

\subsection{Residual Smoothness Mapper}
For the feature post-refinement, we utilize a noise-aware residual module $\mathcal{M}$ to enhance feature robustness while preserving its statistical distribution.
The main process of this mapper can be formulated as:
\begin{equation}
    \tilde{\vz}_m = \tilde{\vz}+\mathcal{M}(\tilde{\vz},\sigma)=\tilde{\vz}+ \sum_{i=0}^{k-1}m_i\left(\tilde{\vz}_i,\sigma\right),
    \label{eq: delta z mapper}
\end{equation}
where $\tilde{\vz}_m$ is the Mapper $\mathcal{M}$ refined feature, $\tilde{\vz}=f_e(\mathcal{P}(\vx+\mathbf{\varepsilon}))$ denotes the purified feature and $\tilde{\vz}_{i+1}=m_i(\tilde{\vz}_i,\sigma)$ is the intermediate output with $\tilde{\vz}_0=\tilde{\vz}$.
$\sigma$ is the noise strength that adaptively controls the output magnitude of the mapper.
$k$ is the number of blocks ($k=3$ in our experiments unless otherwise specified), and each block $m_i(\cdot)$ contains multi-head attention, depthwise convolution, and MLP branches to refine the purified representation.
%
%
To enhance the Gaussian robustness of the refined representation, we introduce the mapper robustness loss $l^\mathcal{M}_\mathrm{rb}$, defined as:
\begin{equation}
l^{\mathcal{M}}_{\mathrm{rb}}
= 1-\mathop{\mathbb{E}}\limits_{\vx \sim \mathcal{D},~ \bm{\varepsilon} \sim \mathcal{N}(0, \sigma^2I)}
\big[\mathrm{Cos}\big(\tilde{\vz}_m , f_e(\vx)\big)\big],
\label{eq:loss smooth for mapper}
\end{equation}
which encourages feature consistency between the refined and clean representations.
Meanwhile, to ensure that the refined feature preserves the statistical characteristics of the clean feature, we introduce two regularization terms: the identical loss $l_\mathrm{id}$ and the statistical loss $l_\mathrm{stats}$, defined as:
\begin{equation}
\left\{
\footnotesize
\begin{aligned}
&l_{\mathrm{stats}}
= \mathop{\mathbb{E}}\limits_{\substack{\vx \sim \mathcal{D},\\ \bm{\varepsilon} \sim \mathcal{N}(0, \sigma^2I)}}
\frac{1}{D}\sum_{d=1}^{D}
\Big[
(\mu_{\tilde{\vz}_m}^{(d)}-\mu_{\vz}^{(d)})^2 + (\sigma_{\tilde{\vz}_m}^{(d)}-\sigma_{\vz}^{(d)})^2
\Big], \\ 
&l_{\mathrm{id}} 
= \mathop{\mathbb{E}}\limits_{\substack{\vx \sim \mathcal{D}}}\|\mathcal{M}(\tilde{\vz},0)\|_2^2. 
\end{aligned}
\right.
\label{eq: loss id and stats}
\end{equation}
where $\tilde{\vz}_m,\vz \in \mathbb{R}^{B\times L \times D}$ with batch size $B$, token number $L$ and feature dimension $D$.
The $l_{\mathrm{stats}}$ enforces consistency between the element-wise mean $\mu^{d}$ and standard deviation $\sigma^{d}$ of two representations, thereby preserving the statistical characteristics.
Meanwhile, the identity loss $l_{\mathrm{id}}$ constrains the mapping network $\mathcal{M}$ when the noise strength $\sigma=0$, promoting stability and preventing undesired distortions on clean inputs.
The overall training loss $\mathcal{L}_\mathcal{M}$ is defined as:
\begin{equation}
\mathcal{L}_\mathcal{M}= l^{\mathcal{M}}_{\mathrm{rb}}+\lambda_2 l_\mathrm{stats} +\lambda_3l_\mathrm{id},
\label{eq:loss for mapper}
\end{equation}
where $\lambda_2$ and $\lambda_3$ are the weighting coefficients.
More details of the training process on the residual smoothness mapper can be found in the supplementary material.

\begin{algorithm}[t]
\caption{Training algorithm of GSB}
\label{alg:PSM}
\small
\begin{algorithmic}[1]
\REQUIRE Dataset $\mathcal{D}$, feature encoder $f_e$, Denoiser $\mathcal{P}$, mapper $\mathcal{M}$,
sampling number $n_0$, noise std $\sigma$, loss weights $\lambda_1,\lambda_2,\lambda_3$.
\STATE \textbf{Stage 1: Train denoiser $\mathcal{P}$.}
\FOR{each batch $\vx \sim \mathcal{D}$}
    \STATE $\vz \leftarrow f_e(\vx)$
    \FOR{$i = 1$ to $n_0$}
        \STATE Sample $\bm{\varepsilon}^i \sim \mathcal{N}(0, \sigma^2 I)$
        \STATE $\tilde{\vx}^i \leftarrow \mathcal{P}(\vx + \bm{\varepsilon}^i)$
        \STATE $\tilde{\vz}^i \leftarrow f_e(\tilde{\vx}^i)$
    \ENDFOR
    \STATE Compute $l^{\mathcal{P}}_{\mathrm{rb}}$ and $l_{\mathrm{mse}}$ using
    Eqs.~\ref{eq: loss mse} and~\ref{eq:loss smooth}.
    \STATE Update $\mathcal{P}$ by gradient descent on $\nabla_{\mathcal{P}} \mathcal{L}_{\mathcal{P}}$
\ENDFOR

\STATE \textbf{Stage 2: Train smoothness mapper $\mathcal{M}$.}
\FOR{each batch $\vx \sim \mathcal{D}$}
    \STATE $\vz \leftarrow f_e(\vx)$
    \FOR{$i = 1$ to $n_0$}
        \STATE Sample $\bm{\varepsilon}^i \sim \mathcal{N}(0, \sigma^2 I)$
        \STATE $\tilde{\vx}^i \leftarrow \mathcal{P}(\vx + \bm{\varepsilon}^i)$
        \STATE $\tilde{\vz}^i \leftarrow f_e(\tilde{\vx}^i)$
        \STATE $\tilde{\vz}^{i}_{m} \leftarrow \tilde{\vz}^i + \mathcal{M}(\tilde{\vz}^i,\sigma)$
    \ENDFOR
    \STATE Compute $l^{\mathcal{M}}_{\mathrm{rb}}$, $l_{\mathrm{id}}$, and $l_{\mathrm{stats}}$ using
    Eqs.~\ref{eq:loss smooth for mapper} and~\ref{eq: loss id and stats}.
    \STATE Update $\mathcal{M}$ by gradient descent on $\nabla_{\mathcal{M}} \mathcal{L}_{\mathcal{M}}$
\ENDFOR
\end{algorithmic}
\end{algorithm}

\subsection{Further Discussion on Plugging GSB}

\textbf{Certified robustness for the encoder $f_e$ with GSB.} 
Let $f'_e$ denote the feature encoder integrated with the proposed GSB.
Then the forward process can be formulated as: 
\begin{equation}
    {f}'_e(\bm{x}+\bm{\varepsilon})= f_e(\mathcal{P}(\vx+\bm{\varepsilon})) + \mathcal{M}(\tilde{\vz},\sigma).
    \label{eq:forward psm eqp f_e}
\end{equation}
Under this condition, the smoothed feature encoder and smoothness score are defined as:
\begin{equation}
\small
\begin{aligned}
&\hat{f}'_e(\bm{x}) = \frac{1}{(2\pi)^{d/2}}
   \int_{\mathbb{R}^d} 
   f'_e(\bm{x} + \bm{\varepsilon})
   \exp\left(-\tfrac{1}{2}\|\bm{\varepsilon}\|^2\right)
   d\bm{\varepsilon}, \\ 
&\hat{S}'(\bm{x}) 
=\tfrac{1}{2}\left(
1+\mathop {\mathbb{E}}\limits_{\bm{\varepsilon}\sim\mathcal N(0, \sigma^2I)}
\left[{Cos}\big(f'_e(\vx+\bm{\varepsilon}),\,f_e(\vx)\big)\right]\right),    
\end{aligned}
\label{eq: score-encoder with PSM}
\end{equation}
where $\hat{S}'(\bm{x}) \in [0,1]$.
We can prove that the Lipschitz property derived in Lemma~\ref{lemma1} still holds, and
the theoretical bound in Section~\ref{sec: feature space smoothing} remains valid.
Then, utilizing Theorem~\ref{theorem:1} and Theorem~\ref{thm:general-target-bound}, we can derive the certified lower bound on $\mathrm{Cos}(\hat{f}'_e(\bm{x'}), f_e(\bm{x}))$ for any adversarial input $\vx'$.

\noindent\textbf{Training algorithm of GSB.}
The training procedure of GSB is summarized in Algorithm~\ref{alg:PSM}, where the denoiser $\mathcal{P}$ and the smoothness mapper $\mathcal{M}$ are trained sequentially via a two-stage manner.
The expectation over Gaussian perturbations is approximated by Monte Carlo sampling with $n_0$ samples drawn from $\mathcal{N}(0,\sigma^2 I)$, where we set $n_0=8$ in our training to balance efficiency and estimation accuracy.

\setcounter{footnote}{1}

\begin{table*}[t]
\renewcommand\arraystretch{0.9}
\centering
\caption{Feature-wise certified protection success rate under different adversarial attack bounds $\epsilon$ and the average certified radius $\mathcal{R}$ for $\mathrm{FCSB} \ge 0.5$.}
\resizebox{0.65\textwidth}{!}{
\begin{tabular}{c|c|c|cccccc|c}
\toprule
\multirow{2}{*}{$\sigma$} & \multirow{2}{*}{Encoder} & \multirow{2}{*}{GSB} &
\multicolumn{6}{c|}{Certified PSR (\%) at adversarial bound $\epsilon$} &
\multirow{2}{*}{Avg. $\mathcal{R}$} \\
\cmidrule{4-9}
& & & 0.125 & 0.25 & 0.375 & 0.50 & 0.625 & 0.75 & \\
\midrule
\multirow{6}{*}{0.25}
& \multirow{2}{*}{CLIP}
    & w/o & 41.0 & 0.1 & 0.0 & 0.0 & 0.0 & 0.0 & 0.118 \\
&   & w/  & 99.0 & 63.0 & 2.0 & 0.0 & 0.0 & 0.0 & 0.267 \\
\cmidrule{2-10}

& \multirow{2}{*}{SigLIP 2}
    & w/o & 59.2 & 2.0 & 0.0 & 0.0 & 0.0 & 0.0 & 0.130 \\
&   & w/  & 99.7 & 81.6 & 10.4 & 0.0 & 0.0 & 0.0 & 0.304 \\
\cmidrule{2-10}

& \multirow{2}{*}{EVA-CLIP}
    & w/o & 16.2 & 0.0 & 0.0 & 0.0 & 0.0 & 0.0 & 0.068 \\
&   & w/  & 93.8 & 53.4 & 2.0 & 0.0 & 0.0 & 0.0 & 0.248 \\
\midrule

\multirow{6}*{0.50}
& \multirow{2}{*}{CLIP}
    & w/o & 35.6 & 2.6 & 0.0 & 0.0 & 0.0 & 0.0 & 0.101 \\
&   & w/  & 99.5 & 94.3 & 72.0 & 37.2 & 9.2 & 0.8 & 0.455 \\
\cmidrule{2-10}

& \multirow{2}{*}{SigLIP 2}
    & w/o & 51.8 & 5.6 & 0.3 & 0.0 & 0.0 & 0.0 & 0.134 \\
&   & w/  & 99.7 & 98.1 & 85.6 & 56.4 & 19.4 & 2.3 & 0.514 \\
\cmidrule{2-10}

& \multirow{2}{*}{EVA-CLIP}
    & w/o & 2.6 & 0.1 & 0.0 & 0.0 & 0.0 & 0.0 & 0.009 \\
&   & w/  & 90.9 & 74.0 & 50.0 & 22.0 & 4.1 & 0.0 & 0.363 \\
\bottomrule
\end{tabular}
}
\label{tab:protection_success_rate}
\end{table*}

\section{Experimental Results on Certified Robustness}
\label{sec: exp_cert}
To evaluate the performance of the proposed FS framework and the GSB module, we consider three vision-language feature encoders as the vanilla encoder $f_e$, including CLIP (ViT-L/14)~\cite{radford2021learning}, SigLIP 2 (So400m/14)~\citep{tschannen2025siglip}, and EVA-CLIP (ViT-B/16)~\citep{sun2023eva}. 
These models span different training objectives and architectural families, allowing us to assess the generality of FS across different models.
Our evaluation comprises three parts: (1) feature-wise certified robustness, which reports the feature-space robustness certification among different certified radii; (2) prediction-wise certified robustness for image classification on the ImageNet Dataset~\citep{deng2009imagenet} and image-text retrieval on the MS COCO Dataset~\citep{lin2014microsoft};
(3) empirical robustness, which reports the empirical performance against white-box feature-space adversarial attacks.


\subsection{Feature-wise Certification}
\label{sec:feature-wise rob}
\subsubsection{Experimental Setup}
All backbone encoders, including CLIP (ViT-L/14)~\citep{radford2021learning}, SigLIP 2 (So400m/14)~\citep{tschannen2025siglip}, and EVA-CLIP (ViT-B/16)~\citep{sun2023eva}, are kept frozen throughout both training and evaluation.
Unless otherwise specified, the global image embedding produced by each encoder is deemed as the protected feature via FS.
For each encoder, we train an independent GSB module in a self-supervised manner using Algorithm~\ref{alg:PSM}. 
For each encoder and noise level, we estimate the Gaussian robustness score using Monte Carlo sampling with $N = 10{,}000$ Gaussian samples per image. 
Considering the huge cost of Monte Carlo sampling, training is performed on 5,000 images sampled from 50 ImageNet classes in the training set. 
Evaluation uses a disjoint set of 2,500 images from the ImageNet validation set, restricted to the same classes.
%

%
\noindent\textbf{Metrics.} We report two feature-wise \textbf{certified metrics}. 1) The certified protection success rate (PSR) under perturbation bound $\epsilon$, defined as the proportion of samples whose $\mathrm{FCSB} > 0.5$ under the specified attack budget. 2) The average certified radius $\mathcal{R}$, defined in Equation~\ref{eq:certified-radius}.
Additionally, we report one feature-wise \textbf{empirical metrics}. 3) empirical PSR at the attack bound $\epsilon$, which evaluates the proportion of samples whose feature cosine similarity remains greater than $0.5$ against actual white-box adversarial examples. 
%

%


%
\begin{figure*}[t]
\centering

\begin{minipage}[t]{0.5\textwidth}
\vspace{0pt}
\centering
\renewcommand{\arraystretch}{1}

\captionof{table}{Empirical protection success rate under different adversarial attacks with perturbation bound $\epsilon$.}
\label{tab:empirical_psr}

\resizebox{\linewidth}{!}{%
\begin{tabular}{c|c|c|cccccc}
\toprule
\multirow{2}{*}{$\sigma$} &
\multirow{2}{*}{Attack} &
\multirow{2}{*}{Encoder} &
\multicolumn{6}{c}{
Empirical PSR (\%) at adversarial bound $\epsilon$} \\
\cmidrule{4-9}
& & & 0.5 & 1 & 2 & 4 & 8 & 16 \\
\midrule

\multirow{6}{*}{0.50}
& \multirow{3}{*}{PGD}
& CLIP w/ GSB
& 100.0 & 100.0 & 100.0 & 98.8 & 79.8 & 34.2 \\
&
& SigLIP 2 w/ GSB
& 100.0 & 100.0 & 100.0 & 97.8 & 80.0 & 30.4 \\
&
& EVA-CLIP w/ GSB
& 95.8 & 73.6 & 38.2 & 8.4 & 0.6 & 0.0 \\
\cmidrule{2-9}

& \multirow{3}{*}{APGD}
& CLIP w/ GSB
& 100.0 & 100.0 & 98.0 & 58.2 & 5.4 & 0.6 \\
&
& SigLIP 2 w/ GSB
& 100.0 & 100.0 & 97.0 & 45.8 & 2.4 & 0.4 \\
&
& EVA-CLIP w/ GSB
& 92.2 & 57.2 & 30.2 & 5.0 & 0.2 & 0.2 \\
\bottomrule
\end{tabular}%
}
\end{minipage}%
\hfill
\begin{minipage}[t]{0.46\textwidth}
\vspace{0pt}
\centering

\includegraphics[width=\linewidth]{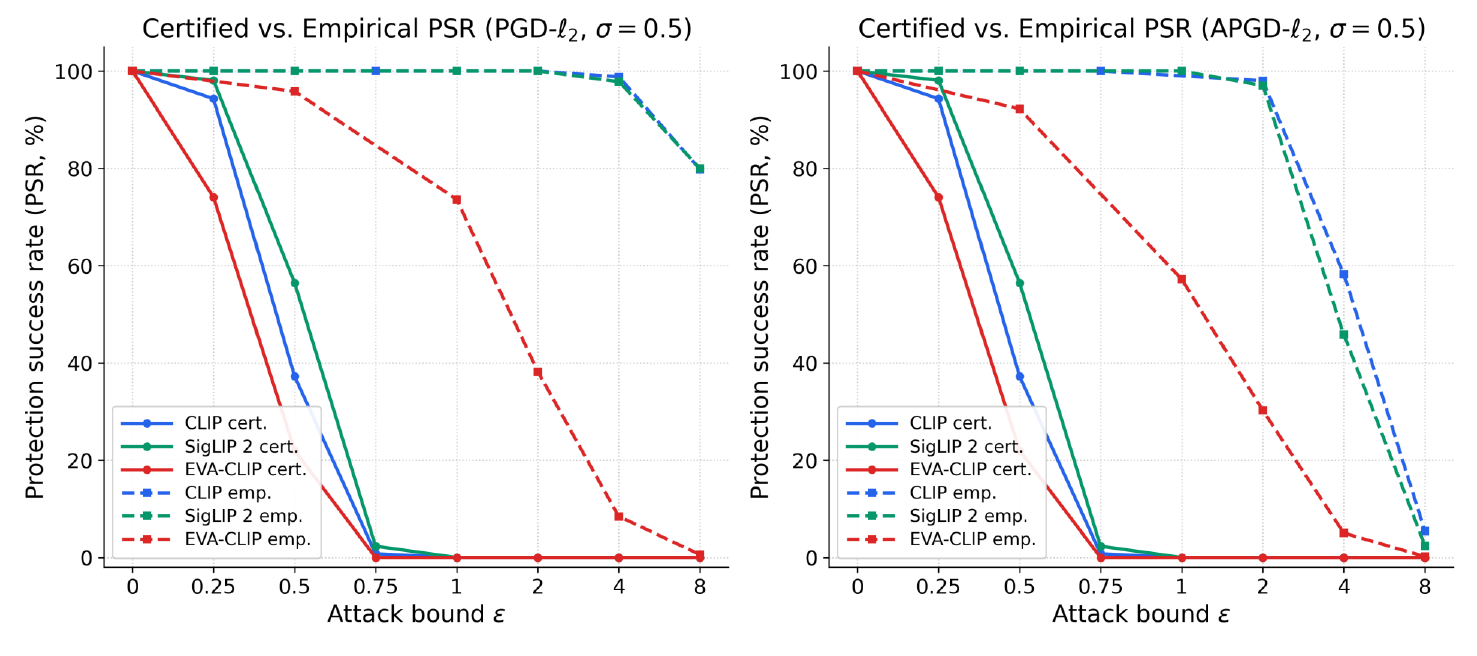}

\captionof{figure}{
Certified and empirical feature-space PSR under different attack bounds $\epsilon$.
}
\label{fig:cert_vs_emp}
\end{minipage}
\end{figure*}




    
    
\subsubsection{Experimental Results}
\textbf{Experimental results on certified robustness.}
The results of the certified feature-space robustness are shown in Table~\ref{tab:protection_success_rate}.
Across all three encoders, the proposed FS framework provides non-trivial certified protection, and incorporating GSB consistently yields substantial improvements.
The gains are particularly pronounced at larger smoothing noise. 
When $\sigma=0.5$, GSB increases the average certified radius from $0.101$ to $0.455$ for CLIP, from $0.134$ to $0.514$ for SigLIP~2, and from $0.009$ to $0.363$ for EVA-CLIP. 
We also observe that the final certified robustness remains encoder-dependent.
Encoders with stronger intrinsic Gaussian robustness, such as SigLIP~2, attain larger certified radii after incorporating GSB. 
This suggests that the effectiveness of the proposed FS framework is fundamentally tied to the encoder's intrinsic Gaussian robustness.
Consequently, improving Gaussian robustness at the encoder level appears to be a promising direction for achieving stronger feature-space certification.

\noindent\textbf{Experimental results on empirical robustness.}
We further evaluate the protection provided by FS under actual white-box adversarial attacks, where PGD-$\ell_2$ and AutoAttack-style adaptive APGD-$\ell_2$ are performed.
APGD-$\ell_2$ follows the AutoAttack-style adaptive step-size strategy and generally provides a stronger optimization baseline.
Both attacks get full white-box access to the protected feature extractor, including the denoiser, mapper, and frozen encoder.
Specifically, given the protected feature extractor, we consider a feature-space attack whose objective is to minimize the cosine similarity between the adversarial feature and the clean feature. 
This attack directly targets the quantity that our feature-wise certification aims to protect.

We evaluate attack bounds of $\epsilon \in {0.5, 1, 2, 4, 8, 16}$ across two adversarial attack methods and three models in the validation set.
This configuration yields a total of 90,000 generated adversarial images.
The empirical results are shown in Table~\ref{tab:empirical_psr}, which broadly align with the certified robustness trends in Table~\ref{tab:protection_success_rate}.
The results show that the proposed FS framework with GSB provides strong practical resistance against feature-space white-box attacks.
For example, under the stronger APGD-$\ell_2$ attack, CLIP and SigLIP~2 both achieve an empirical PSR of $1.000$ at $\epsilon=1$, and PSR values of $0.980$ and $0.970$ at $\epsilon=2$, respectively.
This demonstrates that the certified feature-space protection is not merely a theoretical guarantee, but also translates into strong empirical robustness against direct feature-space attacks.
Meanwhile, in Figure~\ref{fig:cert_vs_emp}, we compare the certified and empirical PSR curves under different attack bounds $\epsilon$. 
The results show that the certified guarantee provides a conservative measure of the empirical robustness.

Overall, the results support two conclusions.
First, the proposed FS framework with GSB not only improves the certified lower bound but also provides practical resistance against strong feature-space white-box attacks.
Second, empirical robustness provided by FS remains closely related to the intrinsic Gaussian robustness of the underlying encoder.
Encoders that obtain stronger certified radii, such as SigLIP~2 and CLIP, also tend to retain higher PSR under direct optimization attacks.
This consistency between certified and empirical evaluations suggests that improving Gaussian feature stability is an effective path toward robust and certifiable vision representations.

\newcolumntype{C}[1]{>{\centering\arraybackslash}p{#1}}

\begin{figure*}[t]
\centering

\begin{minipage}[c]{0.5\textwidth}
\vspace{0pt}
\centering

\renewcommand{\arraystretch}{1}

\captionof{table}{Prediction-wise certified robustness on ImageNet classification under different adversarial attack bounds $\epsilon$.}
\label{tab:prediction_certification_imagenet}

\resizebox{\linewidth}{!}{%
\begin{tabular}{c|c|c|c|*{5}{C{0.65cm}}|c}
\toprule
\multirow{2}{*}{$\sigma$} &
\multirow{2}{*}{Encoder} &
\multirow{2}{*}{Method} &
\multirow{2}{*}{\makecell{Clean\\Acc.}} &
\multicolumn{5}{c|}{
Certified Acc. (\%) at bound $\epsilon$} &
\multirow{2}{*}{Avg. $\mathcal{R}$} \\
\cmidrule{5-9}
& & & & 0 & 0.25 & 0.50 & 0.75 & 1.00 & \\
\midrule

\multirow{6}{*}{0.50}
& \multirow{2}{*}{\makecell{CLIP\\w/ GSB}}
& FS & 95.9 & 86.7 & 79.0 & 72.6 & 63.6 & 50.0 & 0.82 \\
& & RS & 95.6 & 86.8 & 79.0 & 73.3 & 63.2 & 53.0 & 0.83 \\
\cmidrule{2-10}

& \multirow{2}{*}{\makecell{SigLIP 2\\w/ GSB}}
& FS & 97.7 & 91.3 & 85.6 & 81.3 & 76.7 & 64.0 & 1.14 \\
& & RS & 97.6 & 91.3 & 86.2 & 80.9 & 75.0 & 67.6 & 1.23 \\
\cmidrule{2-10}

& \multirow{2}{*}{\makecell{EVA-CLIP\\w/ GSB}}
& FS & 96.6 & 86.4 & 76.3 & 71.3 & 65.3 & 57.3 & 0.91 \\
& & RS & 96.4 & 86.4 & 76.6 & 71.0 & 63.5 & 54.5 & 0.83 \\
\bottomrule
\end{tabular}%
}

\par\medskip

\renewcommand{\arraystretch}{0.9}

\captionof{table}{Empirical prediction accuracy on ImageNet Classification under different adversarial attacks with perturbation bound $\epsilon$.}
\label{tab:empirical_cls}

\resizebox{\linewidth}{!}{%
\begin{tabular}{c|c|c|c|cccccc}
\toprule
\multirow{2}{*}{$\sigma$} &
\multirow{2}{*}{Attack} &
\multirow{2}{*}{Encoder} &
\multirow{2}{*}{Method} &
\multicolumn{6}{c}{
Empirical Acc. (\%) at adversarial bound $\epsilon$} \\
\cmidrule{5-10}
& & & & 0.5 & 1 & 2 & 4 & 8 & 16 \\
\midrule

\multirow{12}{*}{0.50}
& \multirow{6}{*}{PGD}
& \multirow{2}{*}{\makecell{CLIP\\w/ GSB}}
& FS & 86.4 & 86.6 & 86.4 & 86.4 & 86.4 & 85.6 \\
& & & RS & 86.6 & 87.0 & 86.6 & 86.8 & 86.4 & 86.0 \\
\cmidrule{3-10}

& & \multirow{2}{*}{\makecell{SigLIP 2\\w/ GSB}}
& FS & 91.8 & 91.8 & 92.0 & 91.4 & 91.4 & 90.2 \\
& & & RS & 91.6 & 91.8 & 91.8 & 91.4 & 90.6 & 89.6 \\
\cmidrule{3-10}

& & \multirow{2}{*}{\makecell{EVA-CLIP\\w/ GSB}}
& FS & 86.6 & 85.8 & 86.4 & 85.6 & 83.8 & 83.0 \\
& & & RS & 86.2 & 85.6 & 86.2 & 85.4 & 82.4 & 81.6 \\
\cmidrule{2-10}

& \multirow{6}{*}{APGD}
& \multirow{2}{*}{\makecell{CLIP\\w/ GSB}}
& FS & 86.8 & 87.0 & 86.2 & 86.8 & 86.6 & 86.2 \\
& & & RS & 86.8 & 87.0 & 86.2 & 85.4 & 86.6 & 84.8 \\
\cmidrule{3-10}

& & \multirow{2}{*}{\makecell{SigLIP 2\\w/ GSB}}
& FS & 91.8 & 92.0 & 92.2 & 91.4 & 91.4 & 91.2 \\
& & & RS & 91.4 & 91.6 & 91.6 & 91.2 & 91.6 & 89.8 \\
\cmidrule{3-10}

& & \multirow{2}{*}{\makecell{EVA-CLIP w/\\GSB}}
& FS & 86.4 & 86.2 & 86.0 & 85.6 & 85.0 & 82.4 \\
& & & RS & 85.4 & 85.6 & 85.8 & 85.2 & 84.8 & 81.6 \\
\bottomrule
\end{tabular}%
}

\end{minipage}%
\hfill
\begin{minipage}[c]{0.45\textwidth}
\vspace{0pt}
\centering

\includegraphics[width=\linewidth]{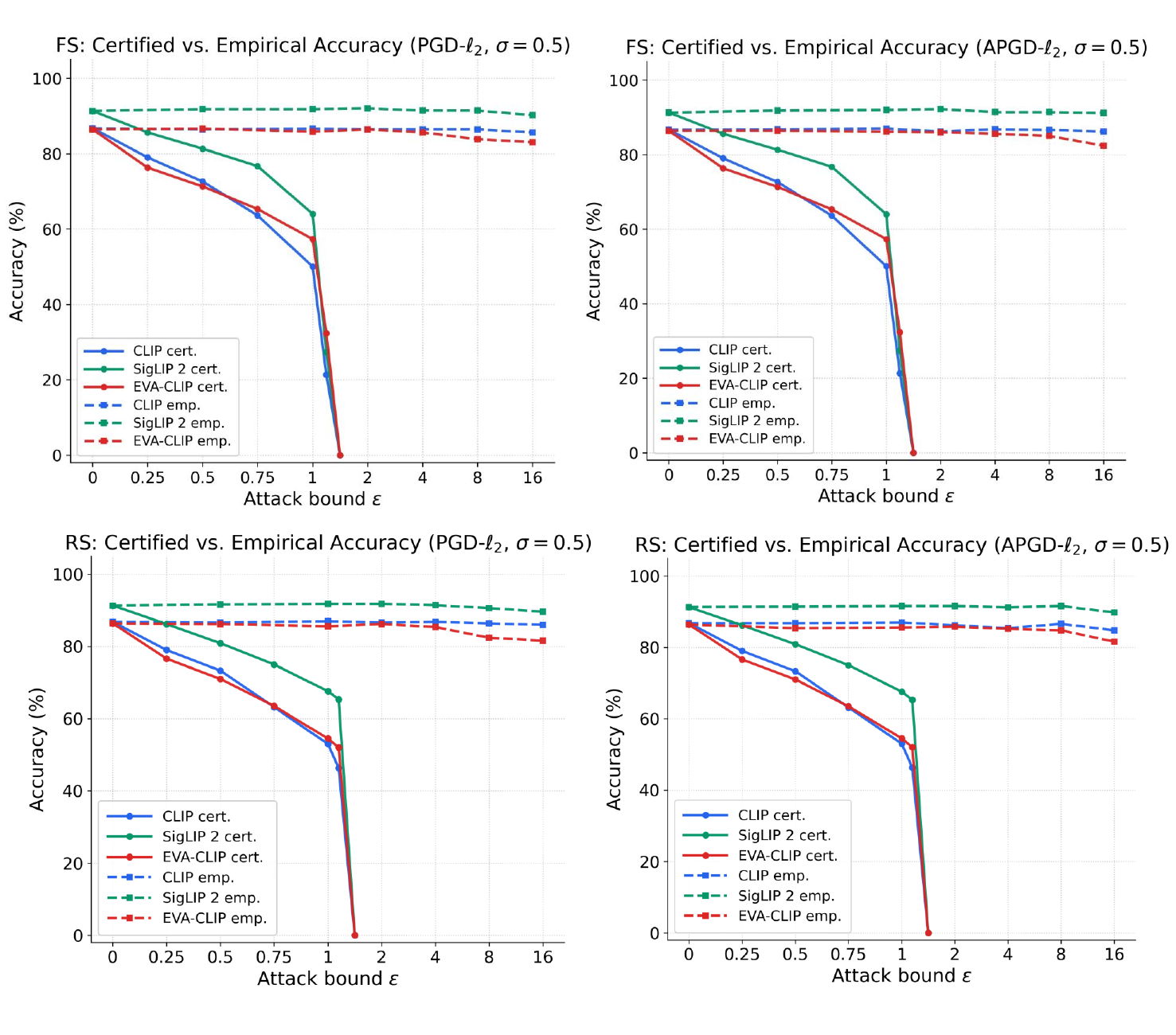}

\captionof{figure}{
Certified and empirical accuracy of FS and RS under different
attack bounds $\epsilon$.
Empirical accuracy is evaluated using adversarial examples generated
in a white-box setting. The results indicate that FS provides
prediction-level protection comparable to that of RS in both
certified and empirical evaluations.
}
\label{fig:cert_vs_emp_cls}

\end{minipage}

\end{figure*}

\subsection{Prediction-wise Certification}
\subsubsection{Experimental Setup}
We further evaluate the prediction-wise certified robustness of FS on two downstream tasks: ImageNet~\citep{deng2009imagenet} classification and image-text retrieval on MS COCO~\citep{lin2014microsoft}.
For both tasks, the backbone vision encoders are kept frozen, and the same GSB modules used in the feature-wise evaluation are adopted.

\noindent\textbf{ImageNet classification.}
For the classification task, we use the training and evaluation sets as claimed in Section~\ref{sec:feature-wise rob}.
To perform the classification tasks, a single-layer cosine prototype classification head is trained on the training data for only one epoch.
Given an image feature $\vz$, the head processes the input feature $\vz$ and produces a normalized prediction vector $\vz_{cls}$. 
The classification score for class $k$ is computed as the cosine similarity between $\vz_{cls}$ and the learned class prototype $\vt_k$: $s_k(\vz_{cls})=\vz_{cls}^\top \vt_k$.  
The certified prediction radius is derived using Equation~\ref{eq:eps-condition-retrieval} by combining the FCSB with the pairwise decision-boundary threshold induced by the learned class prototypes.
This yields a certified radius under which the top-1 class is guaranteed to remain unchanged.

\noindent\textbf{MS COCO image--text retrieval.}
For image--text retrieval, we evaluate the zero-shot top-1 and top-5 retrieval robustness on the MS COCO 5K test split.
Each image is paired with a set of candidate text descriptions, and the retrieval score is computed by the cosine similarity between the protected image feature and the text embedding.
The text embeddings are extracted using the corresponding text encoder when available, and are kept fixed during evaluation.
For each image, the retrieved text is identified according to the largest image-text cosine similarity.
We then compute the pairwise decision-boundary threshold between the top-1/top-5 text and every competing candidate text.
Following the prediction-wise certification rule, the top-1 and top-5 retrieval results are certified if the feature-wise lower bound remains larger than the maximum pairwise threshold among all competitors.
%

\begin{figure*}[t]
\centering

\begin{minipage}[c]{0.5\textwidth}
\vspace{0pt}
\centering

\captionof{table}{FS zero-shot prediction-level certification on COCO image-to-text retrieval. We report clean Top-1 and Top-5 retrieval accuracy and certified retrieval accuracy under different adversarial bounds $\epsilon$.}
\label{tab:prediction_certification_coco}

\resizebox{\linewidth}{!}{%
\begin{tabular}{c|c|c|c|ccccc|c}
\toprule
\multirow{2}{*}{$\sigma$} &
\multirow{2}{*}{Encoder} &
\multirow{2}{*}{Retr.} &
\multirow{2}{*}{\makecell{Clean\\Retr.Acc.}} &
\multicolumn{5}{c|}{
Cert. Retr. Acc. at bound $\epsilon$} &
\multirow{2}{*}{Avg. $\mathcal{R}$} \\
\cmidrule{5-9}
& & & & 0 & 0.125 & 0.25 & 0.375 & 0.50 & \\
\midrule

\multirow{4}{*}{0.125}
& \multirow{2}{*}{\makecell{CLIP\\w/ GSB}}
& Top-1 & 56.1 & 51.9 & 19.6 & 16.8 & 0.1 & 0.0 & 0.068 \\
& & Top-5 & 67.2 & 63.2 & 21.6 & 17.1 & 2.4 & 0.0 & 0.075 \\
\cmidrule{2-10}

& \multirow{2}{*}{\makecell{SigLIP 2\\w/ GSB}}
& Top-1 & 67.7 & 65.4 & 32.4 & 29.2 & 4.8 & 0.0 & 0.115 \\
& & Top-5 & 77.9 & 76.0 & 35.5 & 29.8 & 10.2 & 0.0 & 0.126 \\
\midrule

\multirow{4}{*}{0.25}
& \multirow{2}{*}{\makecell{CLIP\\w/ GSB}}
& Top-1 & 56.1 & 45.2 & 19.6 & 10.8 & 9.3 & 9.3 & 0.100 \\
& & Top-5 & 67.2 & 63.2 & 17.3 & 12.4 & 9.8 & 9.4 & 0.113 \\
\cmidrule{2-10}

& \multirow{2}{*}{\makecell{SigLIP 2\\w/ GSB}}
& Top-1 & 67.7 & 60.1 & 27.3 & 20.6 & 16.3 & 16.0 & 0.141 \\
& & Top-5 & 77.9 & 71.3 & 29.8 & 23.1 & 17.2 & 16.4 & 0.155 \\
\midrule

\multirow{4}{*}{0.50}
& \multirow{2}{*}{\makecell{CLIP\\w/ GSB}}
& Top-1 & 56.1 & 26.8 & 5.8 & 4.4 & 4.0 & 3.6 & 0.035 \\
& & Top-5 & 67.2 & 37.5 & 6.7 & 4.8 & 4.0 & 4.0 & 0.040 \\
\cmidrule{2-10}

& \multirow{2}{*}{\makecell{SigLIP 2\\w/ GSB}}
& Top-1 & 67.7 & 43.9 & 13.0 & 10.5 & 8.7 & 7.1 & 0.103 \\
& & Top-5 & 77.9 & 55.4 & 15.1 & 11.8 & 10.0 & 8.2 & 0.111 \\
\bottomrule
\end{tabular}%
}

\end{minipage}%
\hfill
\begin{minipage}[c]{0.47\textwidth}
\vspace{0pt}
\centering

\includegraphics[width=\linewidth]{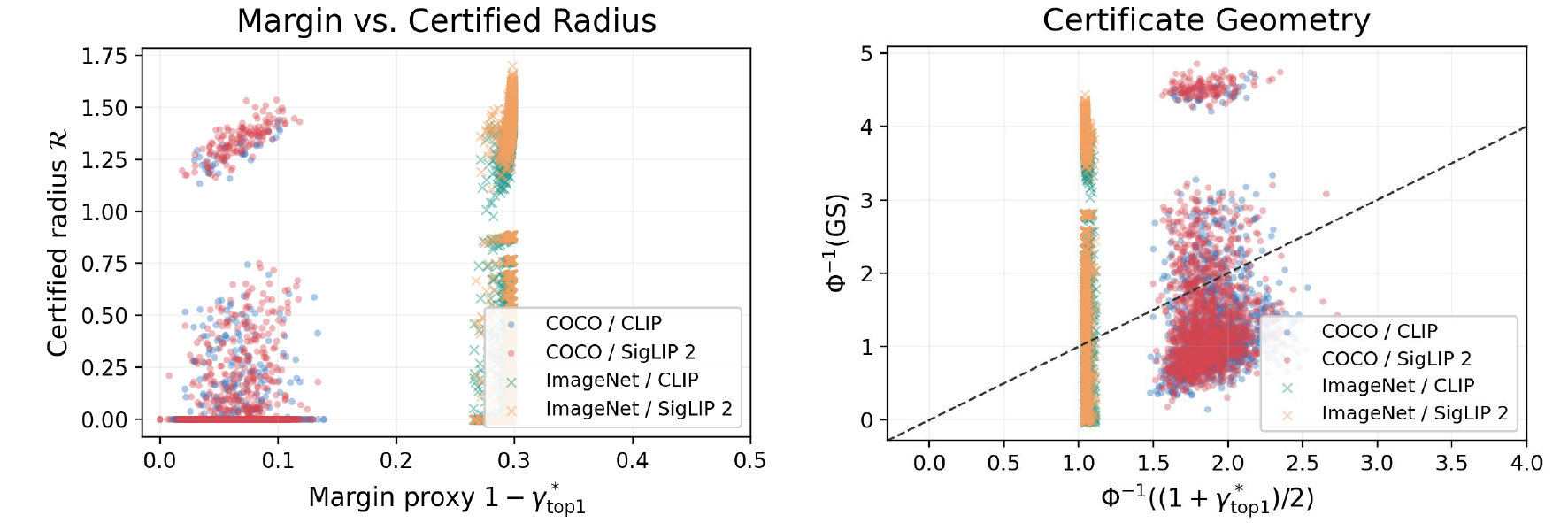}

\captionof{figure}{
Certificate geometry of FS on ImageNet classification and COCO image--text retrieval.
The comparison shows that certified radii are governed by both Gaussian robustness and downstream decision margins; COCO retrieval exhibits tighter margins, leading to weaker certification.
}
\label{fig:cert_cls_vs_ret}

\end{minipage}

\end{figure*}

\noindent\textbf{Baselines.}
We compare FS with randomized smoothing (RS)~\citep{cohen2019certified}, a standard prediction-level certification method, on ImageNet classification.
RS trains a separate smoothed classifier on Gaussian-corrupted inputs and certifies the prediction radius using 10,000 Monte Carlo estimates of the top-class probability.
For a fair comparison, RS uses the same train/test split, the same backbone encoders with the GSB module, and the same Gaussian noise level as FS.

\noindent\textbf{Metrics.}
For both tasks, we report clean accuracy, the certified accuracy under different certified radii, and the average certified radius.
For FS, the certified radius is computed from the Gaussian robustness score and the pairwise decision threshold.
For RS, the certified radius is computed from the lower confidence bound of the top-class probability following the standard randomized smoothing certificate.

\subsubsection{Experimental Results}
\begin{figure*}
\vspace{-2mm}
    \centering
    \includegraphics[width=0.88\linewidth]
    {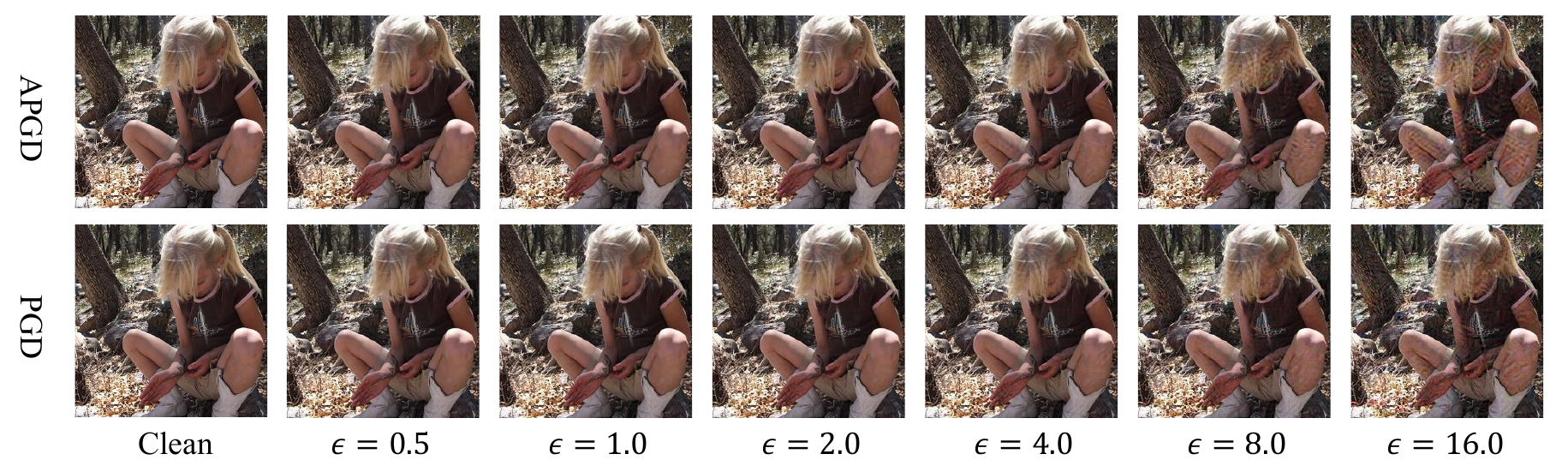}
    \caption{Visualization of adversarial images on ImageNet validation set that were generated by white-box feature-space attack on EVA-CLIP with different attack bounds $\epsilon$.}
    \vspace{-2mm}
    \label{fig:visualization of adv images 1}
\end{figure*}
\noindent\textbf{Results on ImageNet Classification.}
Table~\ref{tab:prediction_certification_imagenet} reports the prediction-wise certified robustness on ImageNet classification under Gaussian noise level $\sigma=0.5$.
For a fair comparison, both FS and RS are evaluated with the same GSB-enhanced feature encoders.
The results show that FS provides non-trivial prediction-level certified robustness and achieves comparable results with RS, which is specifically designed for classification certification.
This indicates that the proposed feature-wise certificate is not limited to preserving feature similarity but can also provide meaningful guarantees for downstream decisions.
Table~\ref{tab:empirical_cls} further evaluates empirical prediction accuracy under white-box feature-space attacks.
The results show that classification accuracy remains stable and high over a wide range of perturbation budgets, even when $\epsilon$ increases to $16$.
The certified--empirical comparison in Figure~\ref{fig:cert_vs_emp_cls} is consistent with this observation: the certified accuracy provides a conservative lower bound of the empirical accuracy.
The visualization of the adversarial images with different attack bounds is shown in Figure~\ref{fig:visualization of adv images 1}.

Comparing Table~\ref{tab:prediction_certification_imagenet} with Table~\ref{tab:protection_success_rate}, we observe that prediction-wise certification is stronger than feature-wise certification.
This comes from the geometry of the certification targets: feature-wise certification requires the representation to remain close to the clean feature, while prediction-wise certification only requires the class ranking induced by the prototypes to remain unchanged.
Thus, feature shifts that reduce clean-feature similarity may still be harmless if they do not cross the relevant decision boundary.
When the clean feature has a sufficient margin over competing prototypes, FS can therefore tolerate non-negligible representation drift while still certifying the top-1 prediction.
In this sense, prediction-wise certification benefits from both the Gaussian stability of the protected encoder and the prediction margin induced by the downstream classifier.


\noindent\textbf{Zero-shot results on MS COCO image--text retrieval.}
The results are shown in Table~\ref{tab:prediction_certification_coco}.
It indicates that FS provides non-trivial prediction-level certificates in the zero-shot retrieval setting, showing that feature-wise robustness can transfer to a challenging open-vocabulary task without training a task-specific retrieval head.
Meanwhile, SigLIP~2 consistently outperforms CLIP, achieving higher clean and certified retrieval accuracy, and average certified radius, suggesting that stronger Gaussian robustness and better vision--language alignment directly benefit certification.
We also observe that Top-5 retrieval is easier to certify than Top-1 retrieval, since preserving a correct caption within the top-5 set imposes a weaker ranking constraint than preserving the exact top-1 prediction.

Compared with ImageNet classification, the certified retrieval accuracy on COCO is lower.
This is mainly due to the intrinsic difficulty of zero-shot retrieval: COCO is certified directly in the frozen image--text embedding space, where candidate captions are dense and semantically overlapping.
As a result, the hardest negative caption can be close to the matched caption, producing a much tighter retrieval decision boundary than the learned ImageNet class prototypes.
Figure~\ref{fig:cert_cls_vs_ret} confirms this certificate geometry: COCO samples concentrate in the small-margin and high-threshold region, while ImageNet samples have substantially larger margins and lower decision thresholds.
Since the FS prediction-wise certificate depends jointly on Gaussian robustness and the downstream decision threshold, these smaller retrieval margins naturally lead to smaller certified radii.

Overall, these results highlight FS as a flexible and interpretable framework for extending feature-wise certification beyond closed-set classification.
They further suggest a promising direction for enhancing certified robustness via FS in a broader vision--language applications: jointly improving Gaussian robustness, cross-modal alignment, and margin-aware decision geometry in the embedding space.
\subsection{Ablation Study}

To rigorously assess the contribution of the proposed FS and each module of GSB, including the denoiser $\mathcal{P}$ and the mapper $\mathcal{M}$, we adopt CLIP-L14 as the base feature encoder and conduct the image classification task following the design in Section~\ref{sec:feature-wise rob}.
We report the feature-wise certified accuracy at different certification radii and the average certified radius $\mathcal{R}$.
The results are summarized in Table~\ref{tab:ablation}.
Compared with vanilla FS, adding either the mapper $\mathcal{M}$ or the denoiser $\mathcal{P}$ consistently improves certified robustness, as reflected by higher certified accuracy across all radii and a larger average certified radius.
Combining both modules achieves the strongest performance, increasing the average certified radius from $0.34$ to $0.82$ and improving the certified accuracy at $\epsilon=1.00$ from $6.6\%$ to $50.0\%$.
These results demonstrate that $\mathcal{P}$ and $\mathcal{M}$ consistently enhance feature-space Gaussian robustness, thereby yielding stronger prediction-wise certified robustness.

\begin{table}[t]
\centering
\caption{Ablation study of FS on certified robustness using CLIP-L/14 with $\sigma=0.50$.}
\label{tab:ablation}

\footnotesize
\setlength{\tabcolsep}{2pt}
\renewcommand{\arraystretch}{1.15}

\begin{tabular*}{\columnwidth}
{@{\extracolsep{\fill}}c|c|cccc|c@{}}
\toprule
\multirow{2}{*}{Method}
& \multirow{2}{*}{GSB}
& \multicolumn{4}{c|}{Certified Acc. (\%) at $\epsilon$}
& \multirow{2}{*}{Avg. $\mathcal{R}$} \\
\cmidrule(lr){3-6}
& & 0.25 & 0.50 & 0.75 & 1.00 & \\
\midrule
\multirow{4}{*}{FS} & None
& 45.3 & 37.0 & 25.3 & 6.6 & 0.34 \\
 & $\mathcal{M}$
& 52.1 & 42.9 & 34.2 & 29.8 & 0.62 \\
 & $\mathcal{P}$
& 70.1 & 61.5 & 52.1 & 46.7 & 0.73 \\
 & $\mathcal{P}+\mathcal{M}$
& 79.0 & 72.6 & 63.6 & 50.0 & 0.82 \\
\bottomrule
\end{tabular*}
\vspace{-3mm}
\end{table}

\section{Experiments Results on Empirical Protection for MLLMs}
\label{sec: exp}
While end-to-end robustness certification for MLLMs is appealing and practically meaningful, smoothing the entire model and certifying its autoregressive predictions with a solvable head are both computationally expensive and challenging.
Alternatively, we seek the feature-wise certification of the visual encoder in MLLMs, and empirically evaluate the resulting task-level robustness under strong white-box adversarial attacks, providing a practical compromise between certification feasibility and effectiveness.

\begin{table*}[t]
\centering
\caption{Experimental results on adversarial robustness of different defense methods and MLLMs on image captioning tasks. Values in parentheses denote the average textual similarity measured by GPTScore. The overall best results are shown in \textbf{bold}, and the best results without smoothing are \underline{underlined}, highlighting the significant performance gain introduced by our proposed FS.}
\label{tab:results image caption}
\resizebox{\linewidth}{!}{
\begin{tabular}{c|cc|ccc|ccc|ccc}
\toprule
\multirow{2}{*}{\textbf{Method}} &
\multicolumn{2}{c|}{\textbf{W/O Attack}} &
\multicolumn{3}{c|}{\textbf{M-Attack}} &
\multicolumn{3}{c|}{\textbf{FOA}} &
\multicolumn{3}{c}{\textbf{AttackVLM}}
\\
\cmidrule(lr){2-3} \cmidrule(lr){4-6}
\cmidrule(lr){7-9} \cmidrule(lr){10-12}
&
\textbf{Cln. ACC$\uparrow$} &
\textbf{Gau. ACC$\uparrow$} &
\textbf{FCS$\uparrow$} & \textbf{ACC$\uparrow$} & \textbf{ASR$\downarrow$} &
\textbf{FCS$\uparrow$} & \textbf{ACC$\uparrow$} & \textbf{ASR$\downarrow$} &
\textbf{FCS$\uparrow$} & \textbf{ACC$\uparrow$} & \textbf{ASR$\downarrow$}
\\
\midrule

\multicolumn{12}{l}{\textbf{LLaVA-1.5-7B}} \\
\midrule

Org. & 92\% (0.75) & 69\% (0.56)
& 0.385 & 1\% (0.06) & 93\% (0.54)
& 0.388 & 1\% (0.05) & 94\% (0.58)
& 0.430 & 3\% (0.09) & 88\% (0.49) \\

Smoothed org. & 92\% (0.74) & 72\% (0.57)
& 0.588 & {63\%} (0.52) & \textbf{3\% (0.03)}
& 0.587 & \textbf{65\% (0.52)} & \textbf{6\% (0.03)}
& 0.598 & \textbf{71\% (0.56)} & \textbf{4\% (0.03)} \\
\cmidrule(lr){1-12}

FARE & 85\% (0.70) & 62\% (0.49)
& 0.588 & 44\% (0.41) & 24\% (0.10)
& 0.504 & 19\% (0.20) & 51\% (0.22)
& 0.499 & 18\% (0.19) & 44\% (0.22) \\

Smoothed FARE & 85\% (0.68) & 73\% (0.61)
& 0.687 & \textbf{64\%} (0.54) & 10\% (0.05)
& 0.634 & 43\% (0.39) & 14\% (0.07)
& 0.653 & 52\% (0.44) & 7\% (0.05) \\
\cmidrule(lr){1-12}

TeCoA & 76\% (0.61) & 46\% (0.41)
& 0.720 & \underline{51\% (0.46)} & \underline{16\% (0.08)}
& 0.674 & \underline{21\% (0.24)} & \underline{37\% (0.19)}
& 0.606 & \underline{17\% (0.18)} & \underline{43\% (0.22)} \\

Smoothed TeCoA & 75\% (0.60) & 63\% (0.52)
& 0.801 & 61\% (0.51) & 17\% (0.08)
& 0.769 & 39\% (0.36) & 10\% (0.05)
& 0.720 & 38\% (0.36) & 14\% (0.07) \\

\midrule
\multicolumn{12}{l}{\textbf{OpenFlamingo-9B}} \\
\midrule

Org. & 81\% (0.62) & 55\% (0.46)
& 0.351 & 1\% (0.09) & 86\% (0.58)
& 0.347 & 1\% (0.09) & 87\% (0.57)
& 0.442 & 16\% (0.22) & 59\% (0.38) \\

Smoothed org. & 78\% (0.59) & 63\% (0.50)
& 0.592 & \textbf{59\% (0.48)} & \textbf{3\% (0.08)}
& 0.588 & \textbf{57\% (0.47)} & \textbf{1\% (0.07)}
& 0.703 & \textbf{61\% (0.51)} & \textbf{0\% (0.05)} \\
\cmidrule(lr){1-12}

FARE & 75\% (0.59) & 47\% (0.38)
& 0.588 & 36\% (0.41) & 17\% (0.15)
& 0.504 & 20\% (0.25) & 40\% (0.24)
& 0.499 & 15\% (0.24) & 40\% (0.25) \\

Smoothed FARE & 69\% (0.56) & 65\% (0.51)
& 0.782 & 48\% (0.46) & 5\% (0.06)
& 0.768 & 35\% (0.35) & 9\% (0.10)
& 0.754 & 40\% (0.40) & 5\% (0.09) \\
\cmidrule(lr){1-12}

TeCoA & 65\% (0.54) & 30\% (0.28)
& 0.720 & \underline{51\% (0.47)} & \underline{13\% (0.12)}
& 0.674 & \underline{21\% (0.24)} & \underline{30\% (0.21)}
& 0.606 & \underline{19\% (0.24)} & \underline{34\% (0.23)} \\

Smoothed TeCoA & 62\% (0.50) & 53\% (0.45)
& 0.835 & 48\% (0.46) & 3\% (0.07)
& 0.818 & 39\% (0.37) & 5\% (0.09)
& 0.805 & 35\% (0.34) & 8\% (0.08) \\

\bottomrule
\end{tabular}}
\end{table*}

\subsection{Experimental Setup}
\noindent\textbf{Evaluated models and tasks.} 
Since the proposed FS requires access to the forward feature computation process of MLLMs, we primarily validate its effectiveness on open-sourced MLLMs, including LLaVA-V1.5-7B~\citep{liu2023visual,liu2024improved} and OpenFlamingo9B~\citep{awadalla2023openflamingo}.
%
We comprehensively assess the performance of plugging the FS under adversarial conditions across multiple downstream tasks, including:
\begin{itemize}
    \item \textbf{Image captioning:} Following~\citep{li2025frustratingly,jia2025adversarial}, we randomly take 100 images from the NIPS 2017 Adversarial Attacks and Defenses Competition dataset~\textsuperscript{\ref{foot:nips2017}} and ask the model to caption the image.
    \item \textbf{Image Classification:} We randomly sample 500 images from 10 classes in the ImageNet dataset~\citep{deng2009imagenet} and ask the model to classify the input.
    \item \textbf{Visual Question Answering (VQA):} We utilize 100 image-question pairs from the ScienceQA dataset~\citep{lu2022learn} and ask the model to select the answer.
\end{itemize}
\footnotetext[\value{footnote}]{\label{foot:nips2017}\url{https://nips.cc/Conferences/2017/CompetitionTrack}}

\noindent\textbf{The threat model.} 
To comprehensively assess the robustness under strong adversaries, we employ three SOTA adversarial attacks specifically designed for MLLMs, named AttackVLM~\citep{zhao2023evaluating}, M-Attack~\citep{li2025frustratingly}, and FOA~\citep{jia2025adversarial} using \textbf{the white-box setting, where attackers can fully access the feature extractor and GSB}.
All attacks are implemented following their original best configurations, with the adversarial perturbation budget $\epsilon$ set to $\|\epsilon\|_{\infty}=16/255$.
%
\begin{table*}[t]
\centering
\caption{Experimental results on adversarial robustness of different defense methods and MLLMs on image classification and VQA tasks. The overall best results are shown in \textbf{bold}, and the best results without smoothing are \underline{underlined}.}
\label{tab:results_cls_vqa}
\resizebox{1\linewidth}{!}{
\begin{tabular}{c|c|cc|ccc|ccc|ccc}
\toprule
\multirow{2}{*}{\textbf{Task}} & \multirow{2}{*}{\textbf{Method}} & \multicolumn{2}{c|}{\textbf{W/O Attack}} & \multicolumn{3}{c|}{\textbf{M-Attack}} & \multicolumn{3}{c|}{\textbf{FOA}} & \multicolumn{3}{c}{\textbf{AttackVLM}} \\
\cmidrule(lr){3-4} \cmidrule(lr){5-7} \cmidrule(lr){8-10} \cmidrule(lr){11-13}
& & \textbf{Cln. ACC$\uparrow$} & \textbf{Gau. ACC$\uparrow$} & \textbf{FCS$\uparrow$} & \textbf{ACC$\uparrow$} & \textbf{ASR$\downarrow$} & \textbf{FCS$\uparrow$} & \textbf{ACC$\uparrow$} & \textbf{ASR$\downarrow$} & \textbf{FCS$\uparrow$} & \textbf{ACC$\uparrow$} & \textbf{ASR$\downarrow$} \\
\midrule
\multicolumn{13}{l}{\textbf{LLaVA-1.5-7B}} \\
\midrule

\multirow{6}{*}{\makecell{Image\\Classification}} & Org. & 92.4\% & 80.4\% & 0.427 & 8.2\% & 78.2\% & 0.437 & 3.8\% & 81.0\% & 0.458 & 6.0\% & 78.4\% \\
& Smoothed org. & 92.4\% & 89.8\% & 0.590 & \textbf{84.8\%} & \textbf{0.2\%} & 0.600 & \textbf{87.2\%} & \textbf{0.2\%} & 0.605 & \textbf{87.0\%} & \textbf{0.2\%} \\
\cmidrule(lr){2-13}
& FARE & 93.4\% & 72.8\% & 0.574 & 55.4\% & 14.6\% & 0.521 & 24.8\% & 39.8\% & 0.508 & 25.8\% & 45.0\% \\
& Smoothed FARE & 93.4\% & 91.8\% & 0.695 & 75.6\% & 5.2\% & 0.681 & 66.2\% & 1.0\% & 0.671 & 68.0\% & 1.0\% \\
\cmidrule(lr){2-13}
& TeCoA & 88.2\% & 59.2\% & 0.731 & \underline{67.4\%} & \underline{2.0\%} & 0.626 & \underline{29.0\%} & \underline{20.0\%} & 0.592 & \underline{33.6\%} & \underline{17.8\%} \\
& Smoothed TeCoA & 88.2\% & 86.8\% & 0.790 & 70.6\% & 0.6\% & 0.714 & 58.8\% & 0.8\% & 0.712 & 56.4\% & 0.4\% \\
\cmidrule(lr){1-13}

\multirow{6}{*}{VQA} & Org. & 54.0\% & 44.0\% & 0.398 & 31.0\% & 28.0\% & 0.383 & 22.0\% & 22.0\% & 0.474 & 25.0\% & 27.0\% \\
& Smoothed org. & 54.0\% & 47.0\% & 0.682 & 47.0\% & \textbf{0.0\%} & 0.643 & \textbf{43.0\%} & \textbf{0.0\%} & 0.712 & \textbf{43.0\%} & 1.0\% \\
\cmidrule(lr){2-13}
& FARE & 49.0\% & 38.0\% & 0.657 & \underline{47.0\%} & 5.0\% & 0.590 & \underline{32.0\%} & \underline{0.0\%} & 0.550 & \underline{38.0\%} & 7.0\% \\
& Smoothed FARE & 50.0\% & 43.0\% & 0.861 & \textbf{48.0\%} & 1.0\% & 0.788 & 38.0\% & 0.0\% & 0.825 & 39.0\% & \textbf{0.0\%} \\
\cmidrule(lr){2-13}
& TeCoA & 38.0\% & 32.0\% & 0.788 & 31.0\% & \underline{2.0\%} & 0.748 & 31.0\% & 2.0\% & 0.667 & 31.0\% & 1.0\% \\
& Smoothed TeCoA & 33.0\% & 35.0\% & 0.925 & 32.0\% & 1.0\% & 0.880 & 34.0\% & 1.0\% & 0.879 & 29.0\% & 1.0\% \\
\bottomrule
\end{tabular}}
\end{table*}
\noindent\textbf{Compared defense methods.}
As robustness certification for MLLMs remains largely unexplored, we primarily compare our method against adversarial training–based defenses, specifically FARE~\citep{schlarmann2024robust} and TeCoA~\citep{maounderstanding}.
Both FARE and TeCoA adopt adversarial training to obtain robust feature encoders that can be directly integrated into models such as LLaVA-1.5-7B and OpenFlamingo-9B.
To ensure a fair and consistent comparison, we obtain all adversarially trained feature encoders from their official repositories, without any modification.

\noindent\textbf{Implementation details.} For practical inference efficiency, we set the number of samples to $n_0=8$ for smoothing, ensuring a favorable trade-off between robustness and runtime.
We train three independent GSB modules for different models, including the encoder of LLaVA-1.5-7B, the encoder of OpenFlamingo-9B, and CLIP-L14.
To further assess cross-model generalization, we directly utilize the GSB trained on a vanilla encoder on adversarially trained encoders, including FARE and TeCoA.
For all feature encoders, we set hyperparameters $\lambda_1,~\lambda_3 = 0.25$, $\lambda = 100$, and $\sigma=0.25$. 
In our tables, the term "smoothed" denotes the process of smoothing the encoder via FS and enhancing it with GSB. 

\noindent\textbf{Evaluation metrics.} We mainly report: the Feature Cosine Similarity ({FCS}), the Accuracy ({ACC}), and the Attack Success Rate ({ASR}).
Specifically, FCS measures the cosine similarity between clean and adversarial features extracted by the same feature encoder, reflecting its feature-wise robustness.
ACC denotes the proportion of correctly completed tasks, while ASR indicates the proportion of cases where the model is successfully manipulated to produce the adversarially targeted outputs.
For image classification and VQA tasks, ACC and ASR are determined by whether the MLLM outputs match the correct or adversarial targets.

For image captioning, following~\citep{li2025frustratingly}, we adopt the LLM-as-a-judge protocol~\citep{zheng2023judging} to evaluate ACC and ASR.
Specifically, we first generate the clean and adversarially targeted captions by feeding the clean and targeted inputs into the vanilla MLLM (\eg, the original LLaVA).
We then obtain the predicted caption by feeding the adversarially perturbed input into the tested model (\eg, Smoothed LLaVA).
The textual similarity is computed using GPTScore~\citep{fu2024gptscore}, where a task is considered successful if the GPTScore between the \textbf{predicted and clean captions is $\mathbf{\ge 0.5}$}, and an attack is considered successful if the GPTScore between the \textbf{predicted and adversarially targeted captions is $\mathbf{\ge 0.3}$}.

\subsection{Results on Different Tasks}
\textbf{Image captioning.} Table~\ref{tab:results image caption} reports the image captioning performance of different MLLMs under clean inputs, Gaussian-corrupted inputs, and adversarial attacks.
The adversaries aim to manipulate the model into generating captions semantically aligned with a maliciously selected target image.
We report both caption accuracy and GPTScore-based textual similarity, where higher ACC and lower ASR indicate stronger robustness.
The results show that converting MLLMs into smoothed variants via the proposed FS smoothing yields \textbf{consistently strong robustness} across diverse attacks, whereas empirical defenses degrade substantially under stronger adversaries.
When attacks escalate from M-Attack to FOA, the accuracy of \textbf{LLaVA with FARE drops from 44\% to 19\%, and that of TeCoA from 51\% to 21\%}. In contrast, FS achieves \textbf{significant robustness gains}, increasing LLaVA accuracy from \textbf{1\% to 65\%} and reducing the ASR from \textbf{94\% to 6\%} under the strongest FOA attack. 
Moreover, FS exhibits strong cross-model generalization, consistently improving the robustness of FARE and TeCoA without additional fine-tuning.
The improvements are especially pronounced under stronger attacks, suggesting that feature-space smoothing provides a more reliable defense than relying solely on empirically robust vision encoders.

\noindent\textbf{Image classification.} 
Table~\ref{tab:results_cls_vqa} presents the classification results, indicating that FS delivers consistent and substantial robustness improvements across all settings.
The adversarial objective is to mislead the model into classifying an adversarial image into a maliciously targeted class.
integrating FS into the vanilla model increases \textbf{LLaVA accuracy from 3.8\% to 87.2\%} and reduces the ASR from \textbf{81\% to 0.2\%} under the strongest FOA attack. 
Similar robustness gains are observed when FS is applied to FARE and TecoA.
Notably, existing MLLMs remain vulnerable to even benign Gaussian perturbations, leading to a clear degradation in clean-image classification accuracy. Specifically, Gaussian noise reduces the accuracy of the vanilla, FARE, and TeCoA encoders from 92.4\%, 93.4\%, and 88.2\% to 80.4\%, 72.8\%, and 59.2\%, respectively. 
By incorporating FS, the models not only achieve stronger adversarial robustness, but also recover much of the performance loss under Gaussian corruption.

\noindent\textbf{VQA.}
The comparative results of different defense methods on LLaVA-1.5-7B on the VQA are presented in Table~\ref{tab:results_cls_vqa}.
%
%
In this setting, the adversarial objective is to mislead the model into selecting a wrong option, \textit{"None of the above"}, for each image–question pair.
This task presents a greater challenge for pure vision-based adversaries than the previous two, as MLLMs can often infer correct answers directly from textual cues.
The results demonstrate that incorporating FS consistently yields {substantial performance improvements} across all attack types, significantly improving prediction accuracy while driving the ASR to nearly zero.
Importantly, adversarially trained robust encoders can compromise utility on challenging tasks such as VQA: FARE reduces clean accuracy from 54.0\% to 49.0\%, and TeCoA further drops it to 38.0\%. In contrast, FS preserves the vanilla model's clean accuracy at 54.0\% and improves its Gaussian-noisy accuracy to 47.0\%, achieving robustness gains without substantial clean-performance degradation.

\begin{table}[t]
\centering
\caption{Experimental results on adversarial robustness of different defense methods on image classification tasks with attack bound $\epsilon=32/255$.}
\label{tab:results image classification large bound}

\setlength{\tabcolsep}{2pt}
\renewcommand{\arraystretch}{0.9}

\begin{tabular}{@{}c|c|ccc@{}}
\toprule
\multirow{2}{*}{\textbf{Model}} &
\multirow{2}{*}{\textbf{Method}} &
\multicolumn{3}{c}{\textbf{FOA}} \\
\cmidrule(lr){3-5}
& &
\textbf{FCS$\uparrow$} &
\textbf{ACC$\uparrow$} &
\textbf{ASR$\downarrow$} \\
\midrule

\multirow{6}{*}{\makecell{LLaVA-\\1.5-7B}}
& Org.
& 0.37 & 3.6\% & 78.6\% \\
& Smoothed org.
& 0.512 & 35.2\% & 24.2\% \\
\cmidrule(lr){2-5}
& FARE
& 0.408 & 10.4\% & 64.2\% \\
& Smoothed FARE
& 0.605 & \textbf{48.4\%} & 19.2\% \\
\cmidrule(lr){2-5}
& TeCoA
& 0.454 & \underline{12.0\%} & \underline{29.6\%} \\
& Smoothed TeCoA
& 0.576 & 32.0\% & \textbf{9.6\%} \\
\bottomrule
\end{tabular}
\vspace{-2mm}
\end{table}
\noindent\textbf{Reasons for the poor performance of FARE and TeCoA despite high FCS.}
Although these encoders preserve high FCS relative to their clean inputs, adversarial training inherently induces a shift in the feature-space distribution. 
Subsequent adversarial perturbations and the limited diversity of adversarial training data further exacerbate this mismatch, leading to pronounced performance degradation on downstream tasks with unseen data distributions.
\begin{table*}[t]
\centering
\caption{Ablation study of GSB components on LLaVA-1.5-7B for image classification under different adversarial attacks. All smoothed variants use Gaussian smoothing with $\sigma=0.25$.}

\label{tab:ablation_gsb_cls}
\resizebox{0.92\linewidth}{!}{
\begin{tabular}{c|cc|ccc|ccc|ccc}
\toprule
\multirow{2}{*}{\textbf{Smoothing}} &
\multirow{2}{*}{\textbf{Denoiser}} &
\multirow{2}{*}{\textbf{Mapper}} &
\multicolumn{3}{c|}{\textbf{M-Attack}} &
\multicolumn{3}{c|}{\textbf{FOA}} &
\multicolumn{3}{c}{\textbf{AttackVLM}}
\\
\cmidrule(lr){4-6} \cmidrule(lr){7-9} \cmidrule(lr){10-12}
& & &
\textbf{FCS$\uparrow$} & \textbf{ACC$\uparrow$} & \textbf{ASR$\downarrow$} &
\textbf{FCS$\uparrow$} & \textbf{ACC$\uparrow$} & \textbf{ASR$\downarrow$} &
\textbf{FCS$\uparrow$} & \textbf{ACC$\uparrow$} & \textbf{ASR$\downarrow$}
\\
\midrule

w/o & \xmark & \xmark
& 0.427 & 8.2\% & 78.2\%
& 0.437 & 3.8\% & 81.0\%
& 0.458 & 6.0\% & 78.4\%
\\
\cmidrule(lr){1-12}

\multirow{4}{*}{FS} & \xmark & \xmark
& 0.527 & 66.2\% & 0.8\%
& 0.535 & 70.4\% & 0.6\%
& 0.534 & 69.4\% & 1.0\%
\\
 & \checkmark & \xmark
& 0.549 & 70.6\% & 1.0\%
& 0.561 & 76.2\% & 0.8\%
& 0.561 & 75.2\% & 0.4\%
\\

& \xmark & \checkmark
& 0.554 & 70.4\% & 1.2\%
& 0.560 & 74.2\% & 0.2\%
& 0.561 & 74.2\% & 0.8\%
\\

& \checkmark & \checkmark
& 0.590 & \textbf{84.8\%} & \textbf{0.2\%}
& 0.600 & \textbf{87.2\%} & \textbf{0.2\%}
& 0.605 & \textbf{87.0\%} & \textbf{0.2\%}
\\

\bottomrule
\end{tabular}}
\end{table*}

\noindent\textbf{Experimental results on large attack bound}
The experimental results under a large adversarial perturbation of $\varepsilon = 32/255$ are presented in Table~\ref{tab:results image classification large bound}.
In addition, Figure~\ref{fig:visualization} provides visualizations of the adversarial examples generated by different attacks and perturbation bounds.
These results show that nearly all existing defenses show great performance drop under such a strong attack.
Under this challenging setting, the best accuracy is obtained by integrating FS with FARE, reaching 48.4\% on LLaVA, while the lowest ASR (9.6\%) is achieved by combining FS with TeCoA.
These findings further highlight the effectiveness and adaptability of the proposed framework, even against large-magnitude adversarial perturbations.

\begin{figure}[t]
    \centering
   \includegraphics[width=0.95\linewidth]{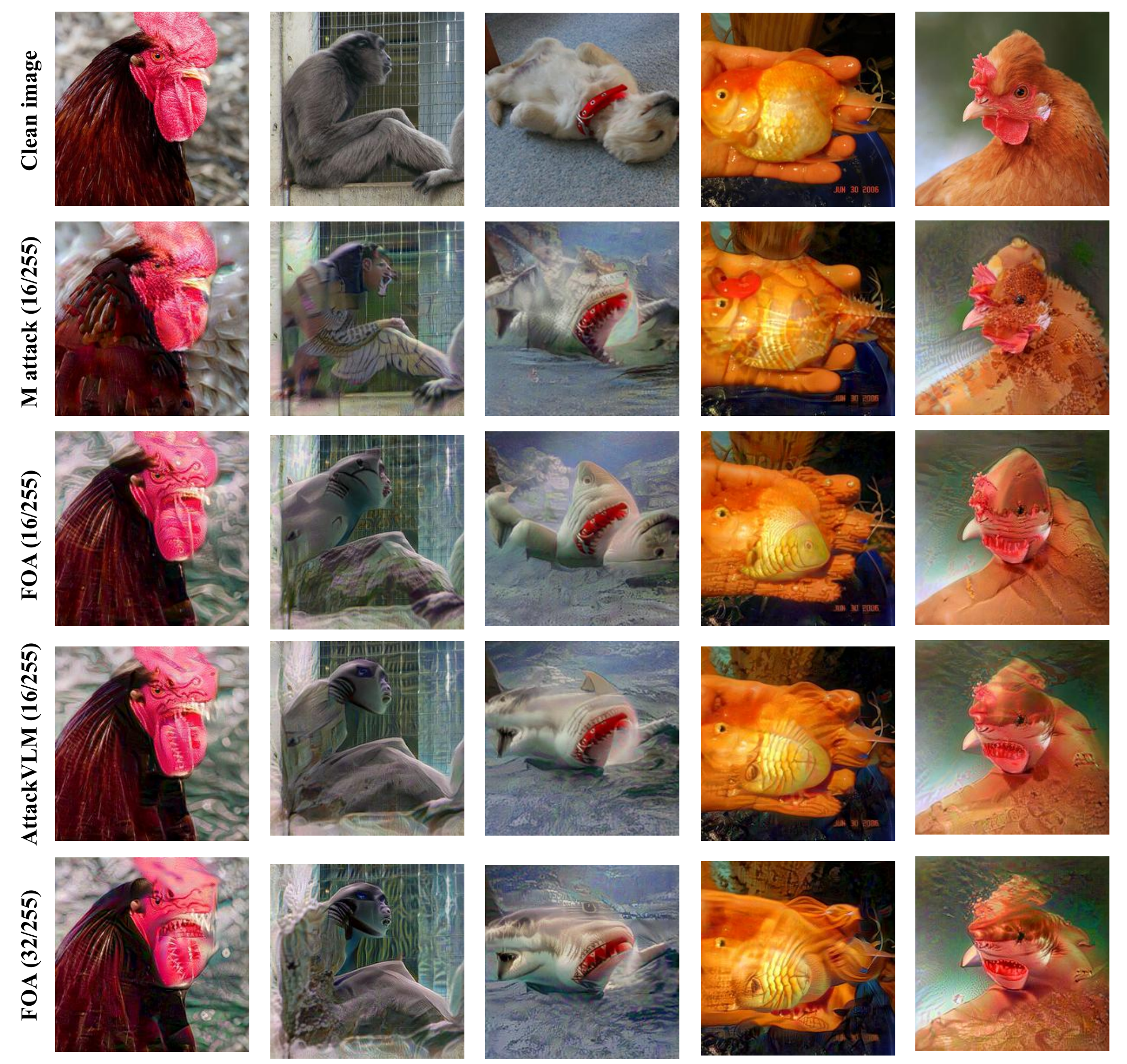}
   \caption{Visualization of the adversarial examples under different attack methods and bounds. The adversarial target is the shark class. }
   \label{fig:visualization}
\end{figure}
\subsection{Ablation Study and Analysis of Efficiency}

\textbf{Ablation:} Table~\ref{tab:ablation_gsb_cls} evaluates the contribution of each GSB component on LLaVA-1.5-7B image classification.
FS alone substantially improves robustness, increasing accuracy from 8.2\%, 3.8\%, and 6.0\% to 66.2\%, 70.4\%, and 69.4\% under M-Attack, FOA, and AttackVLM, respectively.
Adding the denoiser or mapper further improves accuracy, suggesting that both components help stabilize noisy visual features.
The full GSB consistently performs best, achieving 89.8\%, 87.2\%, and 87.0\% accuracy with only 0.2\% ASR across all attacks.
This confirms that the denoiser and mapper provide complementary performance gains.

\begin{table}[t]
\centering
\caption{The analysis of efficiency.}
\renewcommand\arraystretch{1}
\begin{tabular}{c|c|c|c}
\toprule
Model & $n_0$ & Avg. infer time (s) & FCS \\
\midrule
\multirow{4}{*}{\makecell{LLaVA-1.5\\7B}}
& 1  & 0.53 & 0.450 \\
& 4  & 0.64 & 0.562 \\
& 8  & 0.89 & 0.590 \\
& 64 & 4.04 & 0.609 \\
\bottomrule
\end{tabular}
\label{tab:efficiency}
\end{table}

\noindent\textbf{Efficiency analysis:}
Table~\ref{tab:efficiency} presents the efficiency analysis under varying numbers of Gaussian samples, evaluated on a single RTX A6000 GPU.
We consider the image captioning task and report the average per-image inference time together with the average FCS under M-Attack.
Compared to whole-model smoothing, feature-wise smoothing incurs lower inference overhead; notably, setting $n_0=8$ increases the per-image inference time of LLaVA by only $\sim$0.36 seconds, highlighting the practicality of FS for real-world deployment.

\section{Conclusion and Limitations.}
This work studies feature-space certified robustness for vision-language applications.
Building upon Gaussian randomized smoothing, we propose Feature-space Smoothing (FS), a general framework that transforms the original vision encoder into a smoothed variant, which is guaranteed to have a provable Feature Cosine Similarity Bound (FCSB) between clean and adversarial representations.
We further show that the certified feature robustness of FS is governed by the encoder's Gaussian robustness and how such feature-space guarantees can be translated into prediction-wise certified robustness under a cosine similarity measure.
Moreover, we introduce the Gaussian Smoothness Booster (GSB), a plug-and-play module that enhances feature-space Gaussian robustness while preserving feature utility, thereby improving both FCSB and prediction-wise robustness without modifying or re-aligning the backbone encoder.
Extensive experiments across diverse vision encoders and MLLMs demonstrate that FS provides non-trivial certified feature robustness and significantly improves empirical robustness against strong white-box attacks with modest inference overhead.

Despite these promising results, several limitations remain.
First, FS mainly provides certification under $\ell_2$-bounded perturbations; extending the framework to other threat models, such as $\ell_\infty$ perturbations or semantic attacks, is an important future direction.
Second, prediction-wise certification currently relies on cosine-similarity-based decision rules, which are well-suited for retrieval and prototype-based classification but may not directly cover all open-ended MLLM decoding behaviors.
Future work could therefore investigate more general certification mechanisms for autoregressive multimodal generation under broader classes of perturbation constraints.

\backmatter








\section*{Declarations}

\begin{itemize}
\item \textbf{Conflict of interest/Competing interests.} The authors declare that they have no competing interests.
\item \textbf{Data availability.} The datasets used in this study, including ImageNet and MS COCO, are publicly available from their respective sources, subject to applicable access conditions. 
\item Materials availability
\item \textbf{Code availability.} The implementation and evaluation code will be made publicly available upon publication.
\item \textbf{Author contribution.} Song Xia conceived the Feature-Space Smoothing (FS) framework, developed the theoretical foundations and certification methods, designed the Gaussian Smoothness Booster (GSB), implemented the framework, conducted the main experiments, and drafted the manuscript. Meiwen Ding contributed to methodological discussions, adversarial robustness experiments, and manuscript revision. Chenqi Kong contributed to technical discussions, experimental validation, literature review, and manuscript refinement. Wenhan Yang provided guidance on methodological formulation, experimental design, result interpretation, and manuscript improvement. Xudong Jiang supervised and coordinated the research, provided guidance on theoretical development and experimental design, and critically reviewed and revised the manuscript. All authors reviewed and approved the final manuscript.
\end{itemize}
\begin{appendices}
\section{}
\label{secA1}
\subsection{Proof of Theorem~\ref{theorem:1}}
\label{sec: proof of theorem1}
We first show that this Gaussian robustness score $\hat{S}(\vx)$ preserves a good Lipschitz property, which serves as the theoretical foundation for the certified robustness for $\hat{f}_e(\vx)$.
%
\vspace{-4mm}
\begin{lemma}[\textbf{Lipschitz property for the Gaussian robustness score}]
    Let $\Phi(a)=\frac{1}{\sqrt{2\pi}}\int_{-\infty }^{a} \exp(-\frac{1}{2}s^2)ds$ be a standard Gaussian cumulative distribution function and $\Phi^{-1}$ be its inverse. For any feature encoder $f_e: \vx \rightarrow \vz$, with $0<\hat S(\bm x)<1$ , the mapping $\vx \rightarrow \Phi^{-1}(\hat{S}(\bm{x}))$ is  $1/\sigma$-Lipschitz.
    \label{lemma1}
\end{lemma}
\vspace{-4mm}
It implies that the mapping from $\vx$ to $\Phi^{-1}(\hat{S}(\vx))$ exhibits strong adversarial robustness, as it satisfies a $1/\sigma$-Lipschitz constraint.

\begin{proof}[proof of theorem~\ref{theorem:1}]
From Lemma~\ref{lemma1}, we can get:
\begin{equation}
    \frac{\left \| \vx - \vx' \right \| }{\sigma}\ge \Phi^{-1}(\hat{S}_{\vx}(\vx))-\Phi^{-1}(\hat{S}_{\vx}(\vx')).
    \label{eq: prof. therom 1}
\end{equation}
As $\left \| \vx - \vx' \right \| \le \epsilon$, we can get:
\begin{equation}
\Phi^{-1}(\hat{S}_{\vx}(\vx')) \ge \Phi^{-1}(\hat{S}_{\vx}(\vx))-\epsilon/\sigma.   
\label{prof_th1_st2}
\end{equation}
Meanwhile, as $f_e(\vx)$ is a unit vector, using Equation~5 in the main paper, we can get that:
\begin{align}
\hat S(\vx)
&
=\tfrac12\Big(1+\big\langle \E[f_e(\vx+\varepsilon)],\,f_e(\vx)\big\rangle\Big)\notag\\
&=\tfrac12\Big(1+\langle \hat f_e(\vx),\,f_e(\vx)\rangle\Big).
\label{eq:score-linear}
\end{align}
where $\langle\cdot,\cdot \rangle$ is the inner product. Hence, we can get:
\begin{equation}
\label{eq:inner-from-score}
\langle \hat f_e(\vx),\,f_e(\vx)\rangle
=2\hat S(\vx)-1.
\end{equation}
%
As  $\left \| \hat f_e(\vx') \right \| \le 1 $ and $\left \| f_e(\vx) \right \| = 1 $, we can get:
\begin{align}
\label{eq:cos-ge-inner}
{Cos}\big(\hat f_e(\vz),\,f_e(\vx)\big)
&=\frac{\langle \hat f_e(\vz),f_e(\vx)\rangle}{\|\hat f(\vz)\|}
\;\\ \notag
&\ge\;
\langle \hat f_e(\vz),f_e(\vx)\rangle
\end{align}
By the monotonicity of $\Phi$,
\begin{equation}
\label{eq:score-shift}
\hat S_{\vx}(\vx')
\;\ge\;
\Phi\!\Big(\Phi^{-1}\!\big(\hat S(\vx)\big)-\epsilon/\sigma\Big).
\end{equation}
Combining Equations~\ref{eq:inner-from-score},~\ref{eq:cos-ge-inner},~\ref{eq:score-shift}, we thereby can derive that: 
\begin{equation}
    \mathrm{Cos}\big(\hat f(\vx'),\,f_e(\vx)\big)
\;\ge\; 2\,\Phi\!\Big(\Phi^{-1}\!\big(\hat S(\vx)\big)-\epsilon/\sigma\Big)-1,
\label{prof:theorem 1}
\end{equation}
which concludes the proof.
%
\end{proof}
%

\begin{proof}[Proof of Lemma~\ref{lemma1}] (extended from~\cite{salman2019provably}) Note that:
\begin{equation}
\nabla\Phi^{-1}(\hat{S}_{x_c}(\bm{x}))
= \frac{\nabla \hat{S}_{x_c}(\bm{x})}{\Phi'\bigr(\Phi^{-1}(\hat{S}_{x_c}(\bm{x}))\bigr)}.
\label{eq:grad}
\end{equation}
Denote $p=\hat{S}_{\vx_c}(\vx)$, we can get that:
\begin{equation}
    \Phi'\bigr(\Phi^{-1}(\hat{S}_{x_c}(\bm{x}))\bigr)=\frac{1}{\sigma\sqrt{2\pi}}
\exp\!\bigl(-\tfrac12(\Phi^{-1}(p))^2\bigr)
\end{equation}
Thus, we need to prove that for any unit direction $u$.
\begin{equation}
u\cdot\nabla \hat{S}_{\vx_c}(\bm{x})\le
\frac{1}{\sigma\sqrt{2\pi}}
\exp\!\bigl(-\tfrac12(\Phi^{-1}(p))^2\bigr).
\label{eq:condition}
\end{equation}
For the left-hand side, we can get:
\begin{equation}
\footnotesize
\nabla\hat S_{\vx_c}(\vx)=\frac{1}{(2\pi\sigma^2)^{n/2}}\int_{\mathbb R^n}S_{\vx_c}(\vt)\frac{\vt-\vx}{\sigma^2}\exp\!\left(-\frac{\|\vt-\vx\|_2^2}{2\sigma^2}\right)d\vt.
\label{eq:gradient of smoothed score}
\end{equation}
which can also be rewritten as follows:
\[
u\cdot\nabla\hat S_{\vx_c}(\vx)=\frac{1}{\sigma^2}\mathbb E_{\bm{\varepsilon}\sim\mathcal N(0,\sigma^2I)}\left[S_{\vx_c}(\vx+\bm{\varepsilon})(\bm{\varepsilon}\cdot u)\right].
\]
We can now bound the supremum of the above quantity over all measurable score functions $S_{\vx_c}:\mathbb R^n\to[0,1]$, subject to $p=\mathbb E[S_{\vx_c}(\vx+\bm{\varepsilon})]$, by:
\begin{equation}
\begin{aligned}
&\frac{1}{\sigma^2}\mathbb E\left[(\bm{\varepsilon}\cdot u)\mathbf{1}\{\bm{\varepsilon}\cdot u\ge-\sigma\Phi^{-1}(p)\}\right]\\
&\quad=\frac{1}{\sigma\sqrt{2\pi}}\exp\!\Bigl(-\tfrac12(\Phi^{-1}(p))^2\Bigr).
\end{aligned}
\label{eq:proof-of-lemma1}
\end{equation}
Consequently, $\|\nabla\Phi^{-1}(\hat S_{\vx_c}(\vx))\|_2\le 1/\sigma$, which establishes the $1/\sigma$-Lipschitz property.
\end{proof}


\subsection{Pairwise-tight derivation for top-1 image--text retrieval certification}
\label{app:retrieval-cert-proof}
\begin{proof}[Proof of Equation~10]
Fix a competitor $k\neq y$. The pairwise decision boundary between
$\vt_y$ and $\vt_k$ consists of all normalized image features
$\bar\vz\in\mathbb S^{d-1}$ satisfying
\begin{equation}
    \bar\vz^\top \vt_y = \bar\vz^\top \vt_k.
\end{equation}
Equivalently,
\begin{equation}
    \bar\vz^\top(\vt_y-\vt_k)=0.
\end{equation}
Let
\begin{equation}
    \vd_{y,k}=\vt_y-\vt_k.
\end{equation}
Then
\begin{equation}
    \gamma_{y,k}
    =
    \max_{\bar\vz\in\mathbb S^{d-1}}
    \vz^\top \bar\vz
    \quad
    \text{s.t.}\quad
    \bar\vz^\top \vd_{y,k}=0.
\end{equation}

The constraint requires $\bar\vz$ to lie on the unit sphere within the
hyperplane orthogonal to $\vd_{y,k}$. Therefore, the maximizer is obtained
by projecting $\vz$ onto this hyperplane and normalizing. The projection is
\begin{equation}
    \vp
    =
    \vz
    -
    \frac{\vz^\top \vd_{y,k}}
    {\|\vd_{y,k}\|_2^2}
    \vd_{y,k}.
\end{equation}
Hence,
\begin{equation}
    \gamma_{y,k}
    =
    \|\vp\|_2.
\end{equation}
Since $\|\vz\|_2=1$, we have
\begin{equation}
    \|\vp\|_2^2
    =
    1-
    \frac{(\vz^\top \vd_{y,k})^2}
    {\|\vd_{y,k}\|_2^2}.
\end{equation}

Next, using $\vd_{y,k}=\vt_y-\vt_k$, we obtain
\begin{equation}
    \vz^\top \vd_{y,k}
    =
    \vz^\top(\vt_y-\vt_k)
    =
    \vz^\top\vt_y-\vz^\top\vt_k
    =
    a_y-a_k.
\end{equation}
Moreover, since the text embeddings are $\ell_2$-normalized,
\begin{equation}
    \|\vd_{y,k}\|_2^2
    =
    \|\vt_y-\vt_k\|_2^2
    =
    \|\vt_y\|_2^2+\|\vt_k\|_2^2
    -2\vt_y^\top\vt_k
    =
    2(1-\rho_{y,k}).
\end{equation}
Substituting these two identities into the expression for
$\|\vp\|_2^2$ gives
\begin{equation}
    \gamma_{y,k}^2
    =
    1-\frac{(a_y-a_k)^2}{2(1-\rho_{y,k})}.
\end{equation}
Therefore,
\begin{equation}
    \gamma_{y,k}
    =
    \sqrt{
    1-\frac{(a_y-a_k)^2}{2(1-\rho_{y,k})}
    }.
\end{equation}
This completes the proof.
\end{proof}

\begin{proof}[Proof of Corollary~\ref{thm:pred-cert-retrieval}]
Let $\vu=\vz$ and $\vv=\hat\vz'$ denote the normalized clean and normalized smoothed adversarial features, respectively.
By Theorem~\ref{theorem:1}, for any adversarial input $\vx'$ satisfying
\begin{equation}
    \|\vx'-\vx\|_2\le\epsilon,
\end{equation}
the normalized clean and smoothed adversarial features satisfy
\begin{equation}
\label{eq:fcsb-proof}
    \vu^\top \vv
    \ge
    2\Phi\!\left(
    \Phi^{-1}(\hat S(\vx))-\epsilon/\sigma
    \right)-1.
\end{equation}

Suppose that
\begin{equation}
    \epsilon < r_{\mathrm{cert}},
\end{equation}
where
\begin{equation}
    r_{\mathrm{cert}}
    =
    \sigma\left(\Phi^{-1}(\hat S(\vx))
    -
    \Phi^{-1}\!\left(
    \frac{1+\gamma_{\mathrm{top1}}^*}{2}
    \right)\right).
\end{equation}
Then
\begin{equation}
    \Phi^{-1}(\hat S(\vx))-\epsilon/\sigma
    >
    \Phi^{-1}\!\left(
    \frac{1+\gamma_{\mathrm{top1}}^*}{2}
    \right).
\end{equation}
Applying the monotonicity of $\Phi$ to both sides yields
\begin{equation}
    \Phi\!\left(
    \Phi^{-1}(\hat S(\vx))-\epsilon/\sigma
    \right)
    >
    \frac{1+\gamma_{\mathrm{top1}}^*}{2}.
\end{equation}
Multiplying by $2$ and subtracting $1$, we obtain
\begin{equation}
\label{eq:gamma-bound-proof}
    2\Phi\!\left(
    \Phi^{-1}(\hat S(\vx))-\epsilon/\sigma
    \right)-1
    >
    \gamma_{\mathrm{top1}}^*.
\end{equation}
Combining~\eqref{eq:fcsb-proof} and~\eqref{eq:gamma-bound-proof}, we have
\begin{equation}
    \vu^\top \vv > \gamma_{\mathrm{top1}}^*.
\end{equation}

By the definition of
\begin{equation}
    \gamma_{\mathrm{top1}}^*
    =
    \max_{k\neq y}\gamma_{y,k},
\end{equation}
we have
\begin{equation}
    \vu^\top \vv > \gamma_{y,k},
    \qquad
    \forall k\neq y.
\end{equation}
Since $\gamma_{y,k}$ is defined as
\begin{equation}
    \gamma_{y,k}
    =
    \max_{\bar \vv \in \mathbb{S}^{d-1}}
    \vu^\top \bar \vv
    \quad
    \text{s.t.}\quad
    \bar \vv^\top \vt_y = \bar \vv^\top \vt_k,
\end{equation}
every point on the pairwise decision boundary satisfies
\begin{equation}
    \vu^\top \bar \vv \le \gamma_{y,k}.
\end{equation}
Therefore, if
\begin{equation}
    \vu^\top \vv > \gamma_{y,k},
\end{equation}
then $\vv$ cannot lie on the boundary.
Moreover, any continuous path from $\vu$ to a point satisfying
\begin{equation}
    \vv^\top \vt_y \le \vv^\top \vt_k
\end{equation}
must cross the boundary
\begin{equation}
    \bar \vv^\top \vt_y = \bar \vv^\top \vt_k.
\end{equation}
Since all boundary points satisfy $\vu^\top \bar \vv \le \gamma_{y,k}$, such a crossing is impossible whenever $\vu^\top \vv > \gamma_{y,k}$.
Hence,
\begin{equation}
    \vv^\top \vt_y > \vv^\top \vt_k,
    \qquad
    \forall k\neq y.
\end{equation}
Hence the clean top-ranked text remains the unique top-1 retrieval result under the adversarial feature:
\begin{equation}
    \arg\max_k \vv^\top \vt_k = y.
\end{equation}
This completes the proof.
\end{proof}

\subsection{General Prediction-level Cosine Bound}
\label{app:general-prediction-cosine-bound}

\begin{proof}[Proof of Theorem~\ref{thm:general-target-bound}]
For an input $\vx$, let $\vz$ denote the normalized clean prediction,
$\hat\vz'$ denote the normalized adversarial prediction, and $\vt_y$
denote the normalized target prediction. Define
\begin{equation}
    a_y := \vz^\top \vt_y .
\end{equation}
Under the worst-case scenario where the FCSB is attained, we have
\begin{equation}
    \mathrm{Cos}(\hat f_e(\vx'),f_e(\vx))
    =
    \vz^\top \hat\vz'
    =
    \gamma .
\end{equation}

Since $\vz$, $\hat\vz'$, and $\vt_y$ are all normalized, their inner
products can be interpreted as cosine similarities. Let
\begin{equation}
    a_y = \vz^\top \vt_y = \cos\alpha,
    \qquad
    \gamma = \vz^\top \hat\vz' = \cos\Delta,
\end{equation}
where $\alpha$ is the angle between the clean prediction $\vz$ and the
target prediction $\vt_y$, and $\Delta$ is the angle between the clean
prediction $\vz$ and the adversarial prediction $\hat\vz'$.

Thus, in the worst-case scenario, $\hat\vz'$ lies on the spherical
boundary centered at $\vz$ with angular radius $\Delta$. The largest and
smallest possible values of $\hat\vz'^\top \vt_y$ occur when $\hat\vz'$
moves in the two-dimensional plane spanned by $\vz$ and $\vt_y$.
Specifically, $\hat\vz'$ can rotate either toward $\vt_y$ or away from
$\vt_y$ by angle $\Delta$.

Therefore, the angle between $\hat\vz'$ and $\vt_y$ is between the two
extreme cases
\begin{equation}
    \alpha-\Delta
    \qquad\text{and}\qquad
    \alpha+\Delta .
\end{equation}
Hence,
\begin{equation}
    \cos(\alpha+\Delta)
    \le
    \hat\vz'^\top \vt_y
    \le
    \cos(\alpha-\Delta).
\end{equation}
Using the cosine addition formulas, we have
\begin{align}
    \cos(\alpha-\Delta)
    &=
    \cos\alpha\cos\Delta+\sin\alpha\sin\Delta, \\
    \cos(\alpha+\Delta)
    &=
    \cos\alpha\cos\Delta-\sin\alpha\sin\Delta.
\end{align}
Substituting
\begin{equation}
    \cos\alpha=a_y,
    \qquad
    \cos\Delta=\gamma,
\end{equation}
and
\begin{equation}
    \sin\alpha=\sqrt{1-a_y^2},
    \qquad
    \sin\Delta=\sqrt{1-\gamma^2},
\end{equation}
gives
\begin{equation}
\small
    \gamma a_y
    -
    \sqrt{1-\gamma^2}\sqrt{1-a_y^2}
    \le
    \hat\vz'^\top \vt_y
    \le
    \gamma a_y
    +
    \sqrt{1-\gamma^2}\sqrt{1-a_y^2}.
\end{equation}
Equivalently,
\begin{equation}
    \left|
    \hat\vz'^\top \vt_y-\gamma a_y
    \right|
    \le
    \sqrt{(1-\gamma^2)(1-a_y^2)}.
\end{equation}
This completes the proof.
\end{proof}





\end{appendices}


\bibliography{sn-bibliography}

\end{document}

%% file: math_commands.tex

\newcommand{\figleft}{{\em (Left)}}
\newcommand{\figcenter}{{\em (Center)}}
\newcommand{\figright}{{\em (Right)}}
\newcommand{\figtop}{{\em (Top)}}
\newcommand{\figbottom}{{\em (Bottom)}}
\newcommand{\captiona}{{\em (a)}}
\newcommand{\captionb}{{\em (b)}}
\newcommand{\captionc}{{\em (c)}}
\newcommand{\captiond}{{\em (d)}}

\newcommand{\newterm}[1]{{\bf #1}}

\def\figref#1{figure~\ref{#1}}
\def\Figref#1{Figure~\ref{#1}}
\def\twofigref#1#2{figures \ref{#1} and \ref{#2}}
\def\quadfigref#1#2#3#4{figures \ref{#1}, \ref{#2}, \ref{#3} and \ref{#4}}
\def\secref#1{section~\ref{#1}}
\def\Secref#1{Section~\ref{#1}}
\def\twosecrefs#1#2{sections \ref{#1} and \ref{#2}}
\def\secrefs#1#2#3{sections \ref{#1}, \ref{#2} and \ref{#3}}
\def\eqref#1{equation~\ref{#1}}
\def\Eqref#1{Equation~\ref{#1}}
\def\plaineqref#1{\ref{#1}}
\def\chapref#1{chapter~\ref{#1}}
\def\Chapref#1{Chapter~\ref{#1}}
\def\rangechapref#1#2{chapters\ref{#1}--\ref{#2}}
\def\algref#1{algorithm~\ref{#1}}
\def\Algref#1{Algorithm~\ref{#1}}
\def\twoalgref#1#2{algorithms \ref{#1} and \ref{#2}}
\def\Twoalgref#1#2{Algorithms \ref{#1} and \ref{#2}}
\def\partref#1{part~\ref{#1}}
\def\Partref#1{Part~\ref{#1}}
\def\twopartref#1#2{parts \ref{#1} and \ref{#2}}

\def\ceil#1{\lceil #1 \rceil}
\def\floor#1{\lfloor #1 \rfloor}
\def\1{\bm{1}}
\newcommand{\train}{\mathcal{D}}
\newcommand{\valid}{\mathcal{D_{\mathrm{valid}}}}
\newcommand{\test}{\mathcal{D_{\mathrm{test}}}}

\def\eps{{\epsilon}}

\def\reta{{\textnormal{$\eta$}}}
\def\ra{{\textnormal{a}}}
\def\rb{{\textnormal{b}}}
\def\rc{{\textnormal{c}}}
\def\rd{{\textnormal{d}}}
\def\re{{\textnormal{e}}}
\def\rf{{\textnormal{f}}}
\def\rg{{\textnormal{g}}}
\def\rh{{\textnormal{h}}}
\def\ri{{\textnormal{i}}}
\def\rj{{\textnormal{j}}}
\def\rk{{\textnormal{k}}}
\def\rl{{\textnormal{l}}}
\def\rn{{\textnormal{n}}}
\def\ro{{\textnormal{o}}}
\def\rp{{\textnormal{p}}}
\def\rq{{\textnormal{q}}}
\def\rr{{\textnormal{r}}}
\def\rs{{\textnormal{s}}}
\def\rt{{\textnormal{t}}}
\def\ru{{\textnormal{u}}}
\def\rv{{\textnormal{v}}}
\def\rw{{\textnormal{w}}}
\def\rx{{\textnormal{x}}}
\def\ry{{\textnormal{y}}}
\def\rz{{\textnormal{z}}}

\def\rvepsilon{{\mathbf{\epsilon}}}
\def\rvtheta{{\mathbf{\theta}}}
\def\rva{{\mathbf{a}}}
\def\rvb{{\mathbf{b}}}
\def\rvc{{\mathbf{c}}}
\def\rvd{{\mathbf{d}}}
\def\rve{{\mathbf{e}}}
\def\rvf{{\mathbf{f}}}
\def\rvg{{\mathbf{g}}}
\def\rvh{{\mathbf{h}}}
\def\rvu{{\mathbf{i}}}
\def\rvj{{\mathbf{j}}}
\def\rvk{{\mathbf{k}}}
\def\rvl{{\mathbf{l}}}
\def\rvm{{\mathbf{m}}}
\def\rvn{{\mathbf{n}}}
\def\rvo{{\mathbf{o}}}
\def\rvp{{\mathbf{p}}}
\def\rvq{{\mathbf{q}}}
\def\rvr{{\mathbf{r}}}
\def\rvs{{\mathbf{s}}}
\def\rvt{{\mathbf{t}}}
\def\rvu{{\mathbf{u}}}
\def\rvv{{\mathbf{v}}}
\def\rvw{{\mathbf{w}}}
\def\rvx{{\mathbf{x}}}
\def\rvy{{\mathbf{y}}}
\def\rvz{{\mathbf{z}}}

\def\erva{{\textnormal{a}}}
\def\ervb{{\textnormal{b}}}
\def\ervc{{\textnormal{c}}}
\def\ervd{{\textnormal{d}}}
\def\erve{{\textnormal{e}}}
\def\ervf{{\textnormal{f}}}
\def\ervg{{\textnormal{g}}}
\def\ervh{{\textnormal{h}}}
\def\ervi{{\textnormal{i}}}
\def\ervj{{\textnormal{j}}}
\def\ervk{{\textnormal{k}}}
\def\ervl{{\textnormal{l}}}
\def\ervm{{\textnormal{m}}}
\def\ervn{{\textnormal{n}}}
\def\ervo{{\textnormal{o}}}
\def\ervp{{\textnormal{p}}}
\def\ervq{{\textnormal{q}}}
\def\ervr{{\textnormal{r}}}
\def\ervs{{\textnormal{s}}}
\def\ervt{{\textnormal{t}}}
\def\ervu{{\textnormal{u}}}
\def\ervv{{\textnormal{v}}}
\def\ervw{{\textnormal{w}}}
\def\ervx{{\textnormal{x}}}
\def\ervy{{\textnormal{y}}}
\def\ervz{{\textnormal{z}}}

\def\rmA{{\mathbf{A}}}
\def\rmB{{\mathbf{B}}}
\def\rmC{{\mathbf{C}}}
\def\rmD{{\mathbf{D}}}
\def\rmE{{\mathbf{E}}}
\def\rmF{{\mathbf{F}}}
\def\rmG{{\mathbf{G}}}
\def\rmH{{\mathbf{H}}}
\def\rmI{{\mathbf{I}}}
\def\rmJ{{\mathbf{J}}}
\def\rmK{{\mathbf{K}}}
\def\rmL{{\mathbf{L}}}
\def\rmM{{\mathbf{M}}}
\def\rmN{{\mathbf{N}}}
\def\rmO{{\mathbf{O}}}
\def\rmP{{\mathbf{P}}}
\def\rmQ{{\mathbf{Q}}}
\def\rmR{{\mathbf{R}}}
\def\rmS{{\mathbf{S}}}
\def\rmT{{\mathbf{T}}}
\def\rmU{{\mathbf{U}}}
\def\rmV{{\mathbf{V}}}
\def\rmW{{\mathbf{W}}}
\def\rmX{{\mathbf{X}}}
\def\rmY{{\mathbf{Y}}}
\def\rmZ{{\mathbf{Z}}}

\def\ermA{{\textnormal{A}}}
\def\ermB{{\textnormal{B}}}
\def\ermC{{\textnormal{C}}}
\def\ermD{{\textnormal{D}}}
\def\ermE{{\textnormal{E}}}
\def\ermF{{\textnormal{F}}}
\def\ermG{{\textnormal{G}}}
\def\ermH{{\textnormal{H}}}
\def\ermI{{\textnormal{I}}}
\def\ermJ{{\textnormal{J}}}
\def\ermK{{\textnormal{K}}}
\def\ermL{{\textnormal{L}}}
\def\ermM{{\textnormal{M}}}
\def\ermN{{\textnormal{N}}}
\def\ermO{{\textnormal{O}}}
\def\ermP{{\textnormal{P}}}
\def\ermQ{{\textnormal{Q}}}
\def\ermR{{\textnormal{R}}}
\def\ermS{{\textnormal{S}}}
\def\ermT{{\textnormal{T}}}
\def\ermU{{\textnormal{U}}}
\def\ermV{{\textnormal{V}}}
\def\ermW{{\textnormal{W}}}
\def\ermX{{\textnormal{X}}}
\def\ermY{{\textnormal{Y}}}
\def\ermZ{{\textnormal{Z}}}

\def\vzero{{\bm{0}}}
\def\vone{{\bm{1}}}
\def\vmu{{\bm{\mu}}}
\def\vtheta{{\bm{\theta}}}
\def\va{{\bm{a}}}
\def\vb{{\bm{b}}}
\def\vc{{\bm{c}}}
\def\vd{{\bm{d}}}
\def\ve{{\bm{e}}}
\def\vf{{\bm{f}}}
\def\vg{{\bm{g}}}
\def\vh{{\bm{h}}}
\def\vi{{\bm{i}}}
\def\vj{{\bm{j}}}
\def\vk{{\bm{k}}}
\def\vl{{\bm{l}}}
\def\vm{{\bm{m}}}
\def\vn{{\bm{n}}}
\def\vo{{\bm{o}}}
\def\vp{{\bm{p}}}
\def\vq{{\bm{q}}}
\def\vr{{\bm{r}}}
\def\vs{{\bm{s}}}
\def\vt{{\bm{t}}}
\def\vu{{\bm{u}}}
\def\vv{{\bm{v}}}
\def\vw{{\bm{w}}}
\def\vx{{\bm{x}}}
\def\vy{{\bm{y}}}
\def\vz{{\bm{z}}}

\def\evalpha{{\alpha}}
\def\evbeta{{\beta}}
\def\evepsilon{{\epsilon}}
\def\evlambda{{\lambda}}
\def\evomega{{\omega}}
\def\evmu{{\mu}}
\def\evpsi{{\psi}}
\def\evsigma{{\sigma}}
\def\evtheta{{\theta}}
\def\eva{{a}}
\def\evb{{b}}
\def\evc{{c}}
\def\evd{{d}}
\def\eve{{e}}
\def\evf{{f}}
\def\evg{{g}}
\def\evh{{h}}
\def\evi{{i}}
\def\evj{{j}}
\def\evk{{k}}
\def\evl{{l}}
\def\evm{{m}}
\def\evn{{n}}
\def\evo{{o}}
\def\evp{{p}}
\def\evq{{q}}
\def\evr{{r}}
\def\evs{{s}}
\def\evt{{t}}
\def\evu{{u}}
\def\evv{{v}}
\def\evw{{w}}
\def\evx{{x}}
\def\evy{{y}}
\def\evz{{z}}

\def\mA{{\bm{A}}}
\def\mB{{\bm{B}}}
\def\mC{{\bm{C}}}
\def\mD{{\bm{D}}}
\def\mE{{\bm{E}}}
\def\mF{{\bm{F}}}
\def\mG{{\bm{G}}}
\def\mH{{\bm{H}}}
\def\mI{{\bm{I}}}
\def\mJ{{\bm{J}}}
\def\mK{{\bm{K}}}
\def\mL{{\bm{L}}}
\def\mM{{\bm{M}}}
\def\mN{{\bm{N}}}
\def\mO{{\bm{O}}}
\def\mP{{\bm{P}}}
\def\mQ{{\bm{Q}}}
\def\mR{{\bm{R}}}
\def\mS{{\bm{S}}}
\def\mT{{\bm{T}}}
\def\mU{{\bm{U}}}
\def\mV{{\bm{V}}}
\def\mW{{\bm{W}}}
\def\mX{{\bm{X}}}
\def\mY{{\bm{Y}}}
\def\mZ{{\bm{Z}}}
\def\mBeta{{\bm{\beta}}}
\def\mPhi{{\bm{\Phi}}}
\def\mLambda{{\bm{\Lambda}}}
\def\mSigma{{\bm{\Sigma}}}

\newcommand{\tens}[1]{\bm{\mathsfit{#1}}}
\def\tA{{\tens{A}}}
\def\tB{{\tens{B}}}
\def\tC{{\tens{C}}}
\def\tD{{\tens{D}}}
\def\tE{{\tens{E}}}
\def\tF{{\tens{F}}}
\def\tG{{\tens{G}}}
\def\tH{{\tens{H}}}
\def\tI{{\tens{I}}}
\def\tJ{{\tens{J}}}
\def\tK{{\tens{K}}}
\def\tL{{\tens{L}}}
\def\tM{{\tens{M}}}
\def\tN{{\tens{N}}}
\def\tO{{\tens{O}}}
\def\tP{{\tens{P}}}
\def\tQ{{\tens{Q}}}
\def\tR{{\tens{R}}}
\def\tS{{\tens{S}}}
\def\tT{{\tens{T}}}
\def\tU{{\tens{U}}}
\def\tV{{\tens{V}}}
\def\tW{{\tens{W}}}
\def\tX{{\tens{X}}}
\def\tY{{\tens{Y}}}
\def\tZ{{\tens{Z}}}

\def\gA{{\mathcal{A}}}
\def\gB{{\mathcal{B}}}
\def\gC{{\mathcal{C}}}
\def\gD{{\mathcal{D}}}
\def\gE{{\mathcal{E}}}
\def\gF{{\mathcal{F}}}
\def\gG{{\mathcal{G}}}
\def\gH{{\mathcal{H}}}
\def\gI{{\mathcal{I}}}
\def\gJ{{\mathcal{J}}}
\def\gK{{\mathcal{K}}}
\def\gL{{\mathcal{L}}}
\def\gM{{\mathcal{M}}}
\def\gN{{\mathcal{N}}}
\def\gO{{\mathcal{O}}}
\def\gP{{\mathcal{P}}}
\def\gQ{{\mathcal{Q}}}
\def\gR{{\mathcal{R}}}
\def\gS{{\mathcal{S}}}
\def\gT{{\mathcal{T}}}
\def\gU{{\mathcal{U}}}
\def\gV{{\mathcal{V}}}
\def\gW{{\mathcal{W}}}
\def\gX{{\mathcal{X}}}
\def\gY{{\mathcal{Y}}}
\def\gZ{{\mathcal{Z}}}

\def\sA{{\mathbb{A}}}
\def\sB{{\mathbb{B}}}
\def\sC{{\mathbb{C}}}
\def\sD{{\mathbb{D}}}
\def\sF{{\mathbb{F}}}
\def\sG{{\mathbb{G}}}
\def\sH{{\mathbb{H}}}
\def\sI{{\mathbb{I}}}
\def\sJ{{\mathbb{J}}}
\def\sK{{\mathbb{K}}}
\def\sL{{\mathbb{L}}}
\def\sM{{\mathbb{M}}}
\def\sN{{\mathbb{N}}}
\def\sO{{\mathbb{O}}}
\def\sP{{\mathbb{P}}}
\def\sQ{{\mathbb{Q}}}
\def\sR{{\mathbb{R}}}
\def\sS{{\mathbb{S}}}
\def\sT{{\mathbb{T}}}
\def\sU{{\mathbb{U}}}
\def\sV{{\mathbb{V}}}
\def\sW{{\mathbb{W}}}
\def\sX{{\mathbb{X}}}
\def\sY{{\mathbb{Y}}}
\def\sZ{{\mathbb{Z}}}

\def\emLambda{{\Lambda}}
\def\emA{{A}}
\def\emB{{B}}
\def\emC{{C}}
\def\emD{{D}}
\def\emE{{E}}
\def\emF{{F}}
\def\emG{{G}}
\def\emH{{H}}
\def\emI{{I}}
\def\emJ{{J}}
\def\emK{{K}}
\def\emL{{L}}
\def\emM{{M}}
\def\emN{{N}}
\def\emO{{O}}
\def\emP{{P}}
\def\emQ{{Q}}
\def\emR{{R}}
\def\emS{{S}}
\def\emT{{T}}
\def\emU{{U}}
\def\emV{{V}}
\def\emW{{W}}
\def\emX{{X}}
\def\emY{{Y}}
\def\emZ{{Z}}
\def\emSigma{{\Sigma}}

\newcommand{\etens}[1]{\mathsfit{#1}}
\def\etLambda{{\etens{\Lambda}}}
\def\etA{{\etens{A}}}
\def\etB{{\etens{B}}}
\def\etC{{\etens{C}}}
\def\etD{{\etens{D}}}
\def\etE{{\etens{E}}}
\def\etF{{\etens{F}}}
\def\etG{{\etens{G}}}
\def\etH{{\etens{H}}}
\def\etI{{\etens{I}}}
\def\etJ{{\etens{J}}}
\def\etK{{\etens{K}}}
\def\etL{{\etens{L}}}
\def\etM{{\etens{M}}}
\def\etN{{\etens{N}}}
\def\etO{{\etens{O}}}
\def\etP{{\etens{P}}}
\def\etQ{{\etens{Q}}}
\def\etR{{\etens{R}}}
\def\etS{{\etens{S}}}
\def\etT{{\etens{T}}}
\def\etU{{\etens{U}}}
\def\etV{{\etens{V}}}
\def\etW{{\etens{W}}}
\def\etX{{\etens{X}}}
\def\etY{{\etens{Y}}}
\def\etZ{{\etens{Z}}}

\newcommand{\pdata}{p_{\rm{data}}}
\newcommand{\ptrain}{\hat{p}_{\rm{data}}}
\newcommand{\Ptrain}{\hat{P}_{\rm{data}}}
\newcommand{\pmodel}{p_{\rm{model}}}
\newcommand{\Pmodel}{P_{\rm{model}}}
\newcommand{\ptildemodel}{\tilde{p}_{\rm{model}}}
\newcommand{\pencode}{p_{\rm{encoder}}}
\newcommand{\pdecode}{p_{\rm{decoder}}}
\newcommand{\precons}{p_{\rm{reconstruct}}}

\newcommand{\laplace}{\mathrm{Laplace}} 

\newcommand{\E}{\mathbb{E}}
\newcommand{\Ls}{\mathcal{L}}
\newcommand{\R}{\mathbb{R}}
\newcommand{\emp}{\tilde{p}}
\newcommand{\lr}{\alpha}
\newcommand{\reg}{\lambda}
\newcommand{\rect}{\mathrm{rectifier}}
\newcommand{\softmax}{\mathrm{softmax}}
\newcommand{\sigmoid}{\sigma}
\newcommand{\softplus}{\zeta}
\newcommand{\KL}{D_{\mathrm{KL}}}
\newcommand{\Var}{\mathrm{Var}}
\newcommand{\standarderror}{\mathrm{SE}}
\newcommand{\Cov}{\mathrm{Cov}}
\newcommand{\normlzero}{L^0}
\newcommand{\normlone}{L^1}
\newcommand{\normltwo}{L^2}
\newcommand{\normlp}{L^p}
\newcommand{\normmax}{L^\infty}

\newcommand{\parents}{Pa} 

\let\ab\allowbreak